 \documentclass[conference]{IEEEtran}
 \usepackage{times}

\IEEEoverridecommandlockouts                              

\usepackage{algorithm}
\usepackage{algorithmic}

\usepackage{graphics} 
\usepackage{epsfig} 
\usepackage{mathrsfs}
\usepackage{times} 
\usepackage{amsmath} 

\usepackage{amsthm}
\usepackage{amssymb}  
\usepackage{gensymb}

\usepackage[footnotesize]{caption} 

\usepackage{float,bm}
\usepackage{wasysym}
\usepackage[normalem]{ulem}
\usepackage{mathtools}
\usepackage{bbm}
\usepackage{graphicx}
\usepackage{caption}
\usepackage{subcaption}
\usepackage{color}  
\usepackage[dvipsnames]{xcolor}
\usepackage{titlesec}
\usepackage{indentfirst}
\usepackage{multirow}

\usepackage{perl_acronyms}
\usepackage{perl_misc}

\usepackage{color, colortbl}

\usepackage{soul} 

\usepackage{geometry}
\usepackage[
    style=ieee,
    doi=false,
    isbn=false,
    url=false,
    eprint=false,
    backend=bibtex,
    natbib=true
    ]{biblatex}

 \usepackage{nopageno}

\usepackage{hyperref}
\def\hyphenateAndTtWholeString #1{\xHyphenate#1$\wholeString\unskip}
\def\xHyphenate#1#2\wholeString {\if#1$%
    \else\transform{#1}%
    \takeTheRest#2\ofTheString\fi}
\def\takeTheRest#1\ofTheString\fi
{\fi \xHyphenate#1\wholeString}
\def\transform#1{\url{#1}\hskip 0pt plus 1pt}
\def\urlx #1{\href{#1}{\hyphenateAndTtWholeString{#1}}}

\bibliography{references}

\usepackage{hyperref}
\newtheorem{defn}{Definition}
\newtheorem{rem}[defn]{Remark}
\newtheorem{lem}[defn]{Lemma}

\newtheorem{assum}[defn]{Assumption}

\newtheorem{thm}[defn]{Theorem}

\providecommand{\riskmethodname}{\text{RADIUS}}

\newcommand{\new}[1]{{\color{Plum}{#1}}}

\newcommand{\stkout}[1]{{\color{Orange}\ifmmode\text{\sout{\ensuremath{#1}}}\else\sout{#1}\fi}}

\newcommand{\zonocg}[2]{ \text{\textless} #1,\; #2\text{\textgreater}}

\providecommand{\Int}{\texttt{int}}

\providecommand{\zhi}{z^{\text{hi}}}
\providecommand{\zlo}{z^{\text{lo}}}

\providecommand{\dzhi}{\dot z^{\text{hi}}}
\providecommand{\dzlo}{\dot z^{\text{lo}}}
\providecommand{\vlo}{v^{\text{lo}}}
\providecommand{\rlo}{r^{\text{lo}}}

\providecommand{\diag}{\texttt{diag}}

\providecommand{\rot}{\texttt{rotate}}

\providecommand{\cost}{\texttt{cost}}
\providecommand{\Oego}{\mathcal O^\text{ego}}
\providecommand{\vegomax}{\nu^\text{ego}}
\providecommand{\vobsmax}{\nu^\text{obs}}
\providecommand{\tz}{t_0}
\providecommand{\tplan}{t_\text{plan}}
\providecommand{\tnb}{t_\text{m}}
\providecommand{\tf}{t_{f}}
\providecommand{\tm}{t_\text{m}}

\providecommand{\udes}{u^\text{des}}

\providecommand{\ubrk}{u^\text{brake}}

\providecommand{\hdes}{h^\text{des}}
\providecommand{\rdes}{r^\text{des}}

\providecommand{\fhi}{f^{\text{hi}}}
\providecommand{\flo}{f^{\text{lo}}}

\providecommand{\uc}{u^\text{cri}}

\providecommand{\opt}{(\texttt{Opt})}

\providecommand{\PP}{\mathcal{P}}

\providecommand{\W}{\mathcal{W}}

\providecommand{\Z}{\mathcal{Z}}

\renewcommand{\P}{\mathcal{P}}

\providecommand{\Z}{\mathcal{Z}}

\providecommand{\T}{\mathcal{T}}

\providecommand{\I}{\mathcal{I}}
\providecommand{\J}{\mathcal{J}}
\providecommand{\E}{\mathcal{E}}

\renewcommand{\SS}{\mathcal{S}}
\providecommand{\W}{\mathcal{W}}

\providecommand{\R}{\ensuremath \mathbb{R}}
\providecommand{\N}{\ensuremath \mathbb{N}}

\providecommand{\riskmethodname}{\text{Risk-RTD}}

\providecommand{\optE}{(\texttt{Opt-E})}

\providecommand{\prob}{\texttt{Prob}}
\providecommand{\wc}{w^\text{ctr}}

\providecommand{\T}{\mathcal{T}}

\providecommand{\bT}{\overline{\mathcal{S}}}
\providecommand{\Hess}{\texttt{Hess}}

\providecommand{\Gobs}{G^\text{obs}}

\providecommand{\w}{w^{\text{obs}}_{i,j}}

\providecommand{\Oobs}{\mathcal{O}^{\text{obs}}_i}
\providecommand{\wobs}{w^{\text{obs}}_{i,j}}
\providecommand{\pdf}{q_{i,j}}
\providecommand{\pdfT}{q_{i,j,\mathcal S}}

\providecommand{\sig}{\sigma_{i,j}}
\providecommand{\pdfmu}{\mu_{i,j}}

\providecommand{\idx}{\texttt{idx}}

\title{\LARGE \bf RADIUS: Risk-Aware, Real-Time, Reachability-Based Motion Planning}
\author{Jinsun Liu$^{*\dagger}$, Challen Enninful Adu$^{*\dagger}$, Lucas Lymburner$^\dagger$, Vishrut Kaushik$^{\dagger}$,\\
Lena Trang$^\ddagger$, and Ram Vasudevan$^\dagger$
\thanks{$^{*}$These authors contributed equally to this work.}%
\thanks{$^{\dagger}$Robotics, University of Michigan, Ann Arbor, MI. {\tt\small <jinsunl, enninful, llymburn, vishrutk, ramv>@umich.edu}.}%
\thanks{$^{\ddagger}$College of Engineering, University of Michigan, Ann Arbor, MI. {\tt\small ltrang@umich.edu}.}
}
\begin{document}

\setlength{\textfloatsep}{18pt}
\maketitle
\thispagestyle{empty}
\pagestyle{plain}

\begin{abstract}
Deterministic methods for motion planning guarantee safety amidst uncertainty in obstacle locations by trying to restrict the robot from operating in any possible location that an obstacle could be in.
Unfortunately, this can result in overly conservative behavior. 
Chance-constrained optimization can be applied to improve the performance of motion planning algorithms by allowing for a user-specified amount of bounded constraint violation. 
However, state-of-the-art methods rely either on moment-based inequalities, which can be overly conservative, or make it difficult to satisfy assumptions about the class of probability distributions used to model uncertainty. 
To address these challenges, this work proposes a real-time, risk-aware reachability-based motion planning framework called \riskmethodname{}. 
The method first generates a reachable set of parameterized trajectories for the robot offline. 
At run time, \riskmethodname{} computes a closed-form over-approximation of the risk of a collision with an obstacle.
This is done without restricting the probability distribution used to model uncertainty to a simple class (\emph{e.g.}, Gaussian).
Then, 
\riskmethodname{} performs real-time optimization to construct a trajectory
that can be followed by the robot in a manner that is certified to have a risk of collision that is less than or equal to a user-specified threshold.
The proposed algorithm is compared to several state-of-the-art chance-constrained and deterministic methods in simulation, and is shown to consistently outperform them in a variety of driving scenarios. 
A demonstration of the proposed framework on hardware is also provided.
Readers can find the paper project page here\footnote{\urlx{https://roahmlab.github.io/RADIUS/}}. 
\end{abstract}

\section{Introduction}\label{sec:intro}
For mobile robots to safely operate in unstructured environments, they must be able to sense their environments and develop plans to dynamically react to changes while avoiding obstacles. 
Unfortunately, it is challenging to accurately and precisely estimate and predict an obstacle's movement. 
For mobile robots to operate robustly, the uncertainty within these estimations and predictions must be accounted for while generating motion plans.
Various approaches to account for this uncertainty have been proposed in the literature, but these methods either (1) have difficulty being utilized in real-time, (2) make strong assumptions on the class of probability distributions used to model the uncertainty, or (3) generate motion plans that are overly conservative and thereby restrict motion.
This paper develops an algorithm called \textbf{\riskmethodname{}}: \textbf{R}isk-\textbf{A}ware, real-time trajectory \textbf{D}esign \textbf{I}n \textbf{U}ncertain \textbf{S}cenarios (introduced in Figure \ref{fig:RiskRTD-intro}) for \textbf{risk-aware, real-time motion planning} while addressing each of the aforementioned challenges.

We first summarize related algorithms for safe motion planning under uncertainty.
Approaches to address this problem can be categorized as deterministic or stochastic. 
The deterministic approaches often assume some bounded level of uncertainty and solve for motion plans that are robust to this bounded uncertainty. 
Reachability-based methods \cite{RTD}\cite{REFINE} for instance, over-approximate the uncertain regions using polynomial level sets or polytopes and plan paths that do not intersect with these regions. 
As we illustrate in this paper, such methods may be overly conservative as they must over-approximate the uncertain region to include events that may have an exceedingly small probability of occurring to ensure safety.

\begin{figure}[t]
    \centering
    \includegraphics[trim={0cm, 10.5cm, 16.4cm, 0cm},clip,width=1\columnwidth]{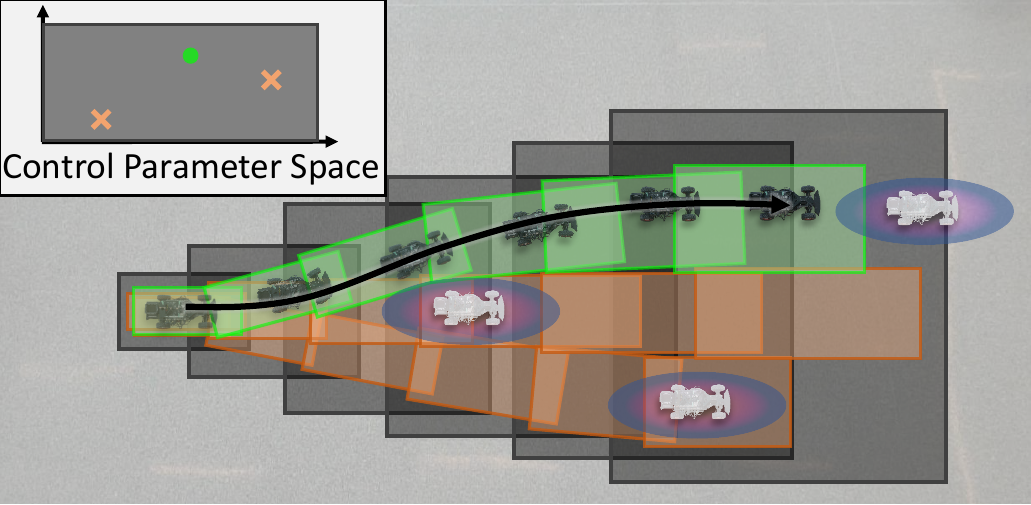}
    \caption{An illustration of the motion planning framework \riskmethodname{} that is developed in this paper. 
    \riskmethodname{} first performs offline reachability analysis using a closed-loop full-order vehicle dynamics to construct a series of control-parameterized, zonotope reachable sets (shown as dark gray boxes) that over-approximate all possible behaviors of the ego vehicle over the planning horizon. 
    During online planning, given some user-defined risk of collision ($\epsilon{}$) for the motion plan, \riskmethodname{} constructs a trajectory by solving an optimization problem that selects subsets of pre-computed zonotope reachable sets that are certified to have a no greater than $\epsilon$ risk of colliding with any obstacles. 
    In this figure, the moving obstacles are shown in white and the 3-$\sigma$ regions of the corresponding probability distributions for the obstacle locations are shown as the purple and blue ellipses where the probability density from low to high is illustrated from blue to purple.
    The subsets of the dark gray zonotope reachable sets corresponding to the trajectory parameter shown in green ensure a collision-free path that is guaranteed to have a no greater than $\epsilon$ risk of colliding with all obstacles, while the two trajectory parameters and their corresponding reachable sets shown orange may have a greater than $\epsilon{}$ risk of collision with the moving obstacles.}
    \label{fig:RiskRTD-intro}
    \vspace*{-0.5cm}
\end{figure}

Stochastic methods incorporate probabilistic information about the environment to reduce the level of conservativeness of their motion plans. 
Many of these methods leverage chance constraints that allow for some user-specified amount of constraint violation. 
For instance, sampling-based algorithms for chance-constrained motion planning apply numerical integration techniques to compute the risk of collision between the agent and obstacles in the environment \cite{janson2018monte,asmussen2007stochastic}. 
In this context, random samples are drawn from the uncertain region, and the number of samples that cause a collision is used to approximate the risk of collision. 
These methods are simple to implement, but can be time-consuming as they require large numbers of samples to converge to a good estimate of the risk of collision.
To address these limitations, moment-based methods upper bound the risk of collision using moments of the probability distribution  \cite{wang2020,wang2fast,cantelli1929sui}. 
Recent works leverage MPC and SOS programming in conjunction with these moment-based upper bounds to perform motion planning. 
Unfortunately, these methods can be conservative as the bound must be valid for any distribution with the given moments.
Coherent risk measures like conditional value-at-risk (CVaR) and Entropic value-at-risk (EVaR) have been used to regulate safety by limiting the expectation of the distance to the safe region using a worst-case quantile representation of a distribution \cite{hakobyanCVAR, ahmadi2021risk, dixit2021risk}.
However, CVaR constraints are difficult to construct directly due to the computationally expensive multi-dimensional integration that is required.
As a result, they are often approximated using sampling methods, which as mentioned earlier can be computationally expensive.
Branch-and-bound methods like the Chance-Constrained Parallel Bernstein Algorithm (CCPBA) use reachability-based methods in conjunction with a branch-and-bound style algorithm to compute tight bounds for the risk of collision \cite{SeanFast}. 
To compute these probabilities they assume that they have access to a cumulative distribution function, \textit{a priori}, that can be evaluated efficiently during online optimization. 
However, such an assumption may not hold for arbitrary distributions.


This paper makes the following contributions.
First, it provides a novel, parallelizable, closed-form over-approximation of the risk of collision given an arbitrary probability distribution with a twice-differentiable density function (Section \ref{subsec:chance constraint}).
Second, it describes the analytical derivative of the probability over-approximation with respect to the control parameter (Section \ref{subsec:gradient}).
Third, this paper presents a general,  real-time risk-aware motion planning framework that guarantees robot safety up to a user-specified risk threshold over time intervals rather than just at discrete time instances (Section \ref{subsec: online operation}).
Lastly, this paper demonstrates \riskmethodname{} in simulation and on hardware and compares its performance to CCPBA, and Cantelli MPC on a variety of scenarios (Section \ref{sec:experiments}).
The rest of this manuscript is organized as follows:
Section \ref{sec: notation} gives necessary notations for the manuscript.
Section \ref{sec: prelim} presents the parameterized vehicle dynamics and environment representation.
Section \ref{sec:online planning} describes obstacle uncertainty, risk-aware vehicle safety and online planning.
Section \ref{sec:risk-rtd and reachability analysis} discusses offline reachability analysis of the vehicle dynamics. 
Section \ref{sec: system_extensions} discusses the generalizability of the proposed algorithm and highlights the assumptions that need to be satisfied to apply RADIUS to other robotic systems.
Section \ref{sec:conclusion} concludes the paper.




\section{Notation}
\label{sec: notation}




This section formalizes notations used throughout the paper.
Sets and subspaces are typeset using calligraphic font.
Subscripts are primarily used as an index or to describe a particular coordinate of a vector.
Let $\R$ and $\N$ denote the spaces of real numbers and natural numbers respectively.
The Minkowski sum between two sets $\mathcal A$ and $\mathcal A'$ is $\mathcal A\oplus \mathcal A' = \{a+a'\mid a\in \mathcal A, ~a'\in \mathcal A'\}$.
Given a set $\mathcal A$, denote its power set as $P(\mathcal A)$.
Given vectors $\alpha,\beta\in\R^n$, let $[\alpha]_i$ denote the $i$-th element of $\alpha$, 
let $\diag(\alpha)$ denote the diagonal matrix with $\alpha$ on the diagonal, and  let $\Int(\alpha,\beta)$ denote the $n$-dimensional box $\{\gamma \in\R^n\mid [\alpha]_i\leq[\gamma]_i\leq[\beta]_i,~\forall i=1,\ldots,n\}$.
Given arbitrary matrix $A\in\R^{n\times n}$, 
and let $\det(A)$ be the determinant of $A$.
Let $\prob(\texttt E)$ denote the probability of occurrence of some event $\texttt E$,
and let $\rot(a)$ denote the 2-dimensional rotation matrix $\begin{bmatrix}
        \cos(a) & -\sin(a) \\ \sin(a) & \cos(a)
    \end{bmatrix}$ for arbitrary $a\in\R$.

Next, we introduce a subclass of polytopes, called zonotopes, that are used throughout this paper:
\begin{defn}
\label{def: zonotope}
A \emph{zonotope} $\Z$ is a subset of $\R^n$ defined as
\begin{equation}
    \mathcal Z = \left\{x\in\R^n\mid x= c+\sum_{k=1}^\ell \beta_k g_k, \quad \beta_k \in [-1,1] \right\}
\end{equation}
with \emph{center} $c\in\R^n$ and $\ell$ \emph{generators} $g_1,\ldots,g_\ell\in\R^n$.
For convenience, we denote $\mathcal Z$ by $\zonocg{c}{G}$ where $G = [g_1, g_2, \ldots, g_\ell ]\in\R^{n\times\ell}$.
\end{defn}
\noindent Note that an $n$-dimensional box is a zonotope because
\begin{equation}
\label{eq: interval is zonotope}
    \Int(\alpha,\beta) = \zonocg{\frac{1}{2}(\alpha+\beta)}{\frac{1}{2}\diag(\beta-\alpha)}.    
\end{equation}
By definition the Minkowski sum of two arbitrary zonotopes  $\Z_1 = \zonocg{c_1}{G_1}$ and $\Z_2=\zonocg{c_2}{G_2}$ is still a zonotope as $\Z_1\oplus\Z_2 = \zonocg{c_1+c_2}{[G_1,G_2]}$.
Note that one can define the multiplication of a matrix $A$ of appropriate size with a zonotope $\Z=\zonocg{c}{G}$ as
\begin{equation}
\label{eq: zono-matrix mult}
   A \Z = \left\{x\in\R^n\mid x= A c+\sum_{k=1}^\ell \beta_k A g_k, ~ \beta_k \in [-1,1] \right\}.
\end{equation}
Note that $A \Z$ is equal to the zonotope $\zonocg{A c}{A G}$. 






\section{Preliminaries}
\label{sec: prelim}

The goal of this work is to plan trajectories for a mobile robot to navigate through environments with uncertainty in the locations of obstacles while having guarantees on the risk of collision. 
We illustrate the proposed method on an autonomous vehicle in this work.
This section discusses the vehicle model, parameterized trajectories, and lastly the environment.

\subsection{Open Loop Vehicle Dynamics}
\label{subsec: dynamics}

\subsubsection{High-Speed Model}
\label{subsubsec: vehicle}
This work adopts the Front-Wheel-Drive vehicle model from \cite{REFINE}, and can be generalized to All-Wheel-Drive and Rear-Wheel-Drive vehicle models.
Let the ego vehicle's states at a time  $t$ be given by $\zhi(t) = [x(t),y(t),h(t),u(t),v(t),r(t)]^\top\in \R^6$, where $x(t)$ and $y(t)$ are the position of the ego vehicle's center of mass in the world frame, $h(t)$ is the heading of the ego vehicle in the world frame, $u(t)$ and $v(t)$ are the longitudinal and lateral speeds of the ego vehicle in its body frame, and $r(t)$ is the yaw rate of the vehicle center of mass.
To simplify exposition, we assume vehicle weight is uniformly distributed and ignore the aerodynamic effect while modeling the flat ground motion of the vehicles by the following dynamics:
\begin{align}
    \label{eq:dynamics_hi}
    \dzhi(t) & = 
    \begin{bmatrix}
        u(t)\cos h(t) - v(t)\sin h(t)\\ u(t)\sin h(t) + v(t)\cos h(t) \\ r(t)\\
        \frac{1}{m}\big(F_\text{xf}(t)+F_\text{xr}(t)\big) + v(t)r(t) + \Delta_u(t) \\ 
        \frac{1}{m}\big(F_\text{yf}(t)+F_\text{yr}(t)\big)-u(t)r(t)+ \Delta_v(t)\\
        \frac{1}{I_\text{zz}} \big(l_\text{f} F_\text{yf}(t) - l_\text{r} F_\text{yr}(t)\big)+ \Delta_r(t)
    \end{bmatrix},
\end{align}
where $l_\text{f}$ and $l_\text{r}$ are the distances from the center of mass to the front and back of the vehicle, $I_\text{zz}$ is the vehicle's moment of inertia, and $m$ is the vehicle's mass.
Note: $l_\text{f}$, $l_\text{r}$, $I_\text{zz}$ and $m$ are all assumed to be known constants.
The tire forces along the longitudinal and lateral directions of the vehicle at time $t$ are $F_\text{xi}(t)$ and $F_\text{yi}(t)$ respectively, where the `i' subscript can be replaced by `f' for the front wheels or `r' for the rear wheels. 
$\Delta_u,\Delta_v,\Delta_r$ are modeling error signals that are unknown and account for imperfect state estimation and tire models.
To ensure the system is well-posed (i.e., its solution exists and is unique), we make the following assumption:
\begin{assum}
\label{ass: dyn error bnd}
    $\Delta_u, \Delta_v,\Delta_r$ are all square-integrable functions and are bounded (i.e., there exist real numbers $M_u, M_v, M_r\in [0,+\infty)$ such that $\|\Delta_u(t)\|_\infty\leq M_u, ~\|\Delta_v(t)\|_\infty\leq M_v, ~\|\Delta_r(t)\|_\infty \leq M_r$ for all $t$).
\end{assum}
\noindent Note that $\Delta_u,\Delta_v,\Delta_r$ can be computed using real-world data \cite[Section IX-C]{REFINE}.
In this work, we assume linear tire models and assume that we can directly control the front tire forces. 
Such assumptions are validated in \cite[Section V.B and VIII.B]{REFINE}.
Additionally, we treat rear tire forces as observed signals. 


\subsubsection{Low-Speed Model}
When the vehicle speed lowers below some critical value $\uc>0$, applying the model described in \eqref{eq:dynamics_hi} becomes intractable as explained in \cite[Section III-B]{REFINE}. 
As a result, in this work when $u(t)\leq \uc$, the dynamics of a vehicle are modeled using a steady-state cornering model \cite[Chapter 6]{gillespie1992fundamentals}, \cite[Chapter 10]{dieter2018vehicle}. 
Note that the critical velocity $\uc$ can be found according to \cite[(5) and (18)]{kim2019advanced}.

The steady-state cornering model or low-speed vehicle model is described using four states, $\zlo(t) = [x(t),y(t),h(t),u(t)]^\top\in\R^4$ at time $t$.
This model ignores transients on lateral velocity and yaw rate. 
Note that the dynamics of $x$, $y$, $h$ and $u$ in the low-speed model are the same as in the high-speed model \eqref{eq:dynamics_hi}; however, the steady-state cornering model describes the yaw rate and lateral speed as
\begin{align}
    \hspace{-0.3cm}\rlo(t) = \frac{\delta(t)u(t)}{l+C_\text{us} u(t)^2}, ~ \vlo(t) =l_\text{r}\rlo(t) - \frac{ml_\text{f}}{\bar c_{\alpha\text{r}}l}u(t)^2\rlo(t)  \label{eq: rlo, vlo}
\end{align}
with understeer coefficient
\begin{equation}
    C_\text{us} = \frac{m}{l}\left( \frac{l_\text{r}}{\bar c_{\alpha \text{f}}} - \frac{l_\text{f}}{\bar c_{\alpha \text{r}}} \right),
\end{equation}
where $\delta(t)$ is the front tire steering angle at time $t$, $\bar c_{\alpha \text{f}}$ and $\bar c_{\alpha \text{r}}$ are cornering stiffness of the front and rear tires respectively. 

As we describe in Section \ref{subsec:cl-dynamics}, the high-speed and low-speed models can be combined together as a hybrid system to describe the vehicle behavior across all longitudinal speeds. 
In short, when $u$ transitions past the critical speed $\uc$ from above at time $t$, the low speed model's states are initialized as:
\begin{equation}
\label{eq: reset h2l}
    \zlo(t) = \pi_{1:4}(\zhi(t))
\end{equation}
where $\pi_{1:4}:\R^6\rightarrow\R^4$ is the projection operator that projects $\zhi(t)$ onto its first four dimensions via the identity relation.
If $u$ transitions past the critical speed from below at time $t$, the high-speed model's states are initialized as
\begin{equation}
\label{eq: reset l2h}
    \zhi(t) = [\zlo(t)^\top, \vlo(t), \rlo(t)]^\top.
\end{equation}

\subsection{Trajectory Parameterization}
\label{subsec: trajectory param}
In this work, each trajectory plan is specified over a compact time interval of a fixed duration $\tf$. 
Because $\riskmethodname{}$ performs receding-horizon planning, we make the following assumption about the time available to construct a new plan:
\begin{assum} \label{assum:tplan}
During each planning iteration starting from time $\tz$, the ego vehicle has $\tplan$ seconds to find a control input that is applied during the time interval $[\tz+\tplan, \tz+\tplan+\tf]$, where $\tf \geq 0$ is some user-specified constant. 
In addition, the vehicle state at time $\tz+\tplan$ is known at time $\tz$. 
\end{assum}
\noindent This assumption requires \riskmethodname{} to generate plans in real-time.
This means that the ego vehicle must create a new plan before it finishes executing its previously planned trajectory.

In each planning iteration, \riskmethodname{} chooses a \emph{desired trajectory} to be followed by the ego vehicle.
The desired trajectory is chosen from a pre-specified continuum of trajectories, with each uniquely determined by an $n_p$-dimensional \textit{trajectory parameter} $p \in \P\subset \R^{n_p}$.
We adapt the definition of trajectory parametrization from \cite[Definition 7]{REFINE}, and
note two important details about the parametrization:
First, all desired trajectories share a time instant $\tnb\in[\tz+\tplan,\tz+\tplan+\tf)$ such that every desired trajectory consists of a \emph{driving maneuver} during $[\tz+\tplan,\tz+\tplan+\tm)$ and a \emph{contingency braking maneuver} during $[\tz+\tplan+\tm,\tz+\tplan+\tf]$.
Second, the contingency braking maneuver slows down the ego vehicle's longitudinal speed to $0$ by $\tf$.
Note this latter property is used to ensure safety as we describe in Section \ref{subsec: risk-aware safety}.
There are many choices of trajectory parametrizations, and the trajectory parametrization utilized in this work is provided in Appendix \ref{app: trajectory parametrization}.

\subsection{Closed Loop Vehicle Dynamics}
\label{subsec:cl-dynamics}
A partial feedback linearization controller that tracks a parameterized desired trajectory robustly and accounts for the modeling error described by $\Delta_u$, $\Delta_v$ and $\Delta_r$ is provided in \cite[Section V]{REFINE}.
Using this controller, the closed loop dynamics of the high and low-speed systems can be written as:  
\begin{align}
    \dzhi(t) &= \fhi(t,\zhi(t),p), \label{eq: hi_veh_closeloop}\\
    \dzlo(t) &= \flo(t,\zlo(t),p), \label{eq: lo_veh_closeloop}
\end{align}
respectively.
Moreover, because the vehicle dynamics change depending on $u$, we model the ego vehicle as a hybrid system $HS$ \cite[Section 1.2]{lunze2009handbook} as is done in \cite[Section V-C]{REFINE}.
The hybrid system $HS$ contains a high-speed mode and a low-speed mode with state $z:=\zhi$, whose dynamics can be written as 
\begin{equation}
\label{eq: dyn tilde_z}
    z(t) = \begin{dcases}
            \fhi(t,\zhi(t),p) , \text{ if } u(t) > \uc, \\ \\
        \begin{bmatrix}
            \flo(t,\zlo(t),p) \\ 0_{2\times1}
        \end{bmatrix} , \text{ if } u(t) \leq \uc,
    \end{dcases}
\end{equation}
Instantaneous transition between the two modes within $HS$ is described using the notion of a \emph{guard} and \emph{reset map}.
The guard triggers a transition and is defined as $\{z(t)\in\R^{6}\mid u(t) =\uc\}$. 
Once a transition happens, the reset map resets the first $z(t)$ via \eqref{eq: reset h2l} if $u(t)$ approaches $\uc$ from above and via \eqref{eq: reset l2h} if $u(t)$ approaches $\uc$ from below.

\subsection{Ego Vehicle, Environment, and Sensing}
\label{subsec: environment}
To provide guarantees about vehicle behavior in a receding horizon planning framework, we define the ego vehicle's footprint similarly to \cite[Definition 10]{REFINE}:
\begin{defn}
\label{def: footprint}
Given $\W\subset\R^2$ as the world space, the ego vehicle is a rigid body that lies in a rectangle $\mathcal O^{ego}:=\Int([-0.5L,-0.5W]^T,[0.5L,0.5W]^T )\subset\W$ with width $W>0$ and length $L>0$ at time $t=0$. 
$\mathcal O^{ego}$ is called the \emph{footprint} of the ego vehicle.
\end{defn}
\noindent For arbitrary time $t$, given state $z(t)$ of the ego vehicle that starts from initial condition $z_0\in\Z_0\subset\R^6$ and applies a control input parameterized by $p\in\P$, the ego vehicle's \emph{forward occupancy} at time $t$ can be represented as
\begin{equation}
    \hspace{-0.2cm}\E\big(t,z_0,p\big) := \rot(h(t))\cdot\Oego + [x(t),y(t)]^\top,
\end{equation}
which is a zonotope by \eqref{eq: zono-matrix mult}.
We define the obstacles as follows:
\begin{defn}
An \emph{obstacle} is a set $\Oobs(t)\subset \W$ that the ego vehicle should not collide with at time $t$, where $i\in\I$ is the index of the obstacle and $\I$ contains finitely many elements. 
\end{defn}

\noindent The dependency on $t$ in the definition of an obstacle allows the obstacle to move as $t$ varies. 
However, if the $i$-th obstacle is static, then $\Oobs(t)$ is a constant.
Note that in this work, we assume that we do not have perfect knowledge of obstacle locations and motion as is normally the case in real-life scenarios.
We assume that this uncertainty in the locations of obstacles is represented by some arbitrary probability distribution.
This is described in the next section.
Assuming that the ego vehicle has a maximum speed $\vegomax$ and all obstacles have a maximum speed $\vobsmax$ for all time, we make the following assumption on planning and sensing horizon. 
\begin{assum}
\label{ass: sense horizon}
The ego vehicle senses all obstacles within a sensor radius greater than $(\tf +\tplan)\cdot(\vegomax+\vobsmax)+0.5\sqrt{L^2+W^2}$ around its center of mass.
\end{assum}
\noindent Assumption \ref{ass: sense horizon} ensures that any obstacle which may cause a collision between times $t\in[\tz+\tplan, \tz+\tplan+\tf]$ can be detected by the vehicle \cite[Theorem 15]{vaskovtowards}.
Note one could treat sensor occlusions as obstacles that travel at the maximum obstacle speed \cite{yu2019occlusion,yu2020risk}.
To simplify notation, we reset time to $0$ whenever a feasible control policy is about to be applied, \emph{i.e.}, $t_0+\tplan = 0$.
Finally to aid in the descriptions of obstacle uncertainty and system trajectory over-approximation in Sections \ref{sec:online planning} and \ref{sec:risk-rtd and reachability analysis}, we partition the planning horizon $[0,\tf]$ into $\tf/\Delta_t$ \emph{time intervals} with some positive number $\Delta_t$ that divides $\tf$, and denote $\T_j$ the $j$-th time interval $[(j-1)\Delta_t, j\Delta_t]$ for any $j\in \J:=\{1,2,\ldots,\tf/\Delta_t\}$.

\section{Online Planning Under Uncertainty}
\label{sec:online planning}

This section constructs a chance-constrained optimization problem to generate risk-aware motion plans. 
First, we describe the representations of the obstacle uncertainty used in this work. 
Then, we define risk-aware vehicle safety and conclude by illustrating how to formulate online planning as a chance-constrained optimization that limits the risk of collision.

\subsection{Obstacle Uncertainty}
\label{subsec:obs uncertainty}
To incorporate uncertainty in both obstacle sensing and obstacle motion prediction into the motion planning framework, let $\wobs$ be a random variable that takes values in $\W$.
$\wobs$ describes the possible locations of the center of the $i$-th obstacle, $\Oobs(t)$, for any time $t\in \T_j$.
We then make the following assumption about how $\wobs$ is distributed:
\begin{assum}
\label{ass: obs pdf}
    For any $i\in\I$ and $j\in\J$, $\wobs$'s 
    Probability Density Function (PDF), $\pdf:\W\rightarrow[0,+\infty)$, exists and is twice-differentiable.
    In addition, there exists some $\wobs$ sampled according to $\pdf$ such that $\Oobs(t)\subseteq\zonocg{\wobs}{\Gobs}$ at some $t \in \T_j$, where $\Gobs$ is a 2-row constant matrix.
\end{assum}



\noindent According to Assumption \ref{ass: obs pdf}, $\zonocg{\wobs}{\Gobs} = \wobs+\zonocg{0}{\Gobs}$ has an uncertain center $\w$ with probability density $\pdf(\w)$, and has invariant shape and size with respect to $i$ and $j$ due to the constant generator matrix $\Gobs$ that accounts for the footprint of any obstacle.
Such probability density functions $\pdf$ can be generated by, for example, performing variants of Kalman Filter \cite{vel_EKF_zindler,almeida2013real} on the $i$-th obstacle given its dynamics or detecting the $i$-th obstacle with a Bayesian confidence framework \cite{fisac2018probabilistically} during $\T_j$.
The generator matrix $\Gobs$ can be generated as the union of footprints of all obstacles.


\subsection{Risk-Aware Vehicle Safety}
\label{subsec: risk-aware safety}
In dynamic environments, it may be difficult to avoid collisions in all scenarios (\emph{e.g.}, a parked ego vehicle can be run into). 
As a result, we instead develop a trajectory synthesis technique that ensures that the ego vehicle is not-at-fault \cite[Definition 11]{vaskovtowards}.
We define not-at-fault safety in this work from a probabilistic perspective by bounding the probability of the ego vehicle running into any obstacles.
\begin{defn}\label{defn:notatfault-risk}
Let the ego vehicle start from initial condition $z_0\in\Z_0$ with control parameter $p\in\P$.
Given a user-specified \emph{allowable risk threshold} $\epsilon\in[0,1]$, the ego vehicle is \emph{not-at-fault} with a risk of collision of at most $\epsilon$, if it is stopped, or if 
\begin{equation}
    \label{eq: chance constraint def}
     \sum_{i\in\I}\sum_{j\in\J}\int_{ \cup_{t\in \T_j} \left( \E(t,z_0,p)\oplus\zonocg{0}{\Gobs} \right)} \pdf(w) ~dw \leq \epsilon,
\end{equation}
while it is moving during $[0,\tf]$.
\end{defn}
\noindent Note that the domain of integration in \eqref{eq: chance constraint def} is the ego vehicle's forward occupancy buffered by the obstacle's footprint.
Thus to satisfy \eqref{eq: chance constraint def}, the probability that the buffered ego vehicle's forward occupancy intersects with an obstacle's center must be bounded by epsilon over all time intervals and all possible obstacles.
In particular, this ensures that the probability of the ego vehicle (including its footprint) intersecting with any obstacle (including its footprint) over the planning horizon is bounded by $\epsilon$ via \cite[Lem. 5.1]{guibas2003zonotopes}.

\subsection{Online Optimization}
\label{subsec:online opt}


To construct a motion plan that ensures the ego vehicle is not-at-fault with a risk of collision at most $\epsilon$ during online planning, one could solve the following chance-constrained optimization:
\begin{align*}
    \min_{p \in \P} & ~ \cost(z_0,p) \hspace{5.3cm} \opt\\
    \text{s.t.}
    &  \sum_{i\in\I}\sum_{j\in\J}\int_{ \cup_{t\in \T_j} \left( \E(t,z_0,p)\oplus\zonocg{0}{\Gobs} \right)} \pdf(w) ~dw \leq \epsilon
\end{align*}
where $\cost:\Z_0\times\P \to \R$ is a user-specified cost function, and the constraint is a chance constraint that ensures the vehicle is not-at-fault with a maximum risk of collision $\epsilon$ during the planning horizon as stated in Definition \ref{defn:notatfault-risk}.

To achieve real-time motion planning, \opt{} must be solved within $\tplan$ seconds.
However, efficiently evaluating the chance constraint in \opt{} in a closed form can be challenging in real applications for two reasons. 
First, the exact information of the ego vehicle's location $\E(t,z_0,p)$ at any time is usually inaccessible due to the nonlinear and hybrid nature of the vehicle dynamics.
Second, the probability distribution that describes the uncertain observation of an obstacle's location as described in Assumption \ref{ass: obs pdf} can be arbitrary.
As illustrated in Appendix \ref{subsec: ablation study}, even approximating this chance constraint accurately using Monte-Carlo is challenging for real-time motion planning. 
Therefore to achieve real-time performance, \riskmethodname{} seeks a closed-form approximation of the risk of collision for evaluation efficiency.
This paper focuses on constructing an over-approximation to ensure that \riskmethodname{} does not under-estimate the true risk of collision or generate plans that violate the not-at-fault condition from Definition \ref{defn:notatfault-risk}. 
Moreover, it is preferable that this approximation is differentiable, as providing the gradient of the constraint can speed up the solving procedure of online optimization.
We describe how we generate an approximation that satisfies these requirements in the next two sections.

\section{Offline Reachability Analysis}
\label{sec:risk-rtd and reachability analysis}

Due to the nonlinear and hybrid nature of the vehicle dynamics, it is challenging to compute the trajectory of vehicle state exactly when evaluating the chance constraint in \opt{}.
To resolve this challenge, \riskmethodname{} over-approximates the ego vehicle's trajectory using zonotopes as stated below:
\begin{assum}
\label{ass: offline reachability}
    Let $z$ be a solution to \eqref{eq: dyn tilde_z} starting from initial condition $z_0\in\Z_0$ with control parameter $p\in\PP$.
    For each $j\in\J$, there exists a map $\xi_j:\Z_0\times\PP\rightarrow P(\W)$ such that 
    \begin{enumerate}
        \item $\xi_j(z_0,p)$ contains the ego vehicle's footprint during $\T_j$, i.e., $\cup_{t\in \T_j} \E(t,z_0,p)\subseteq \xi_j(z_0,p)$, and
        \item $\xi_j(z_0,p)$ is a zonotope of the form $\zonocg{c_j(z_0)+A_j\cdot p}{G_j}$ with some linear function $c_j:\Z_0\rightarrow \R^2$, some matrix $A_j\in\R^{2\times n_p}$ and some 2-row matrix $G_j$.
    \end{enumerate}
\end{assum}

The collection of maps $\{\xi_j\}_{j\in\J}$ can be constructed by applying existing techniques for offline reachability analysis \cite[Section VI]{REFINE}.
In particular, $\{\xi_j\}_{j\in\J}$ can be generated by first applying the open-source toolbox CORA \cite{althoff2015introduction} to over-approximate the trajectory of the ego vehicle's state with initial condition $z_0$ and control parameter $p$ using a collection of zonotopes, where each zonotope over-approximates a segment of the trajectory during one small time interval among $\{\T_j\}_{j\in\J}$, and then accounts for the ego vehicle's footprint.
The proof of Lemma 26 in \cite{REFINE} provides explicit formulas for $\xi_j$, $c_j$, $A_j$ and $G_j$.
Note formulas provided in \cite[Lemma 26]{REFINE} assume the computation in the body frame of the ego vehicle, i.e., assuming $x(0)=y(0)=h(0)=0$.
In the case when the initial position and heading of the ego vehicle are not zeros, one can represent $\xi_j(z_0,p)$ in the world frame via frame transformation based on $z_0$.
In the remainder of this manuscript, we assume $\xi_j(z_0,p)$ is represented in the world frame. 
We refer to $\xi_j(z_0,p)$ as the zonotope reachable set.
\begin{figure*}[!htb]
    \centering
    \begin{subfigure}[b]{0.24\textwidth}
         \includegraphics[trim={0cm, 18.4cm, 27.9cm, 0cm},clip,width=\textwidth]{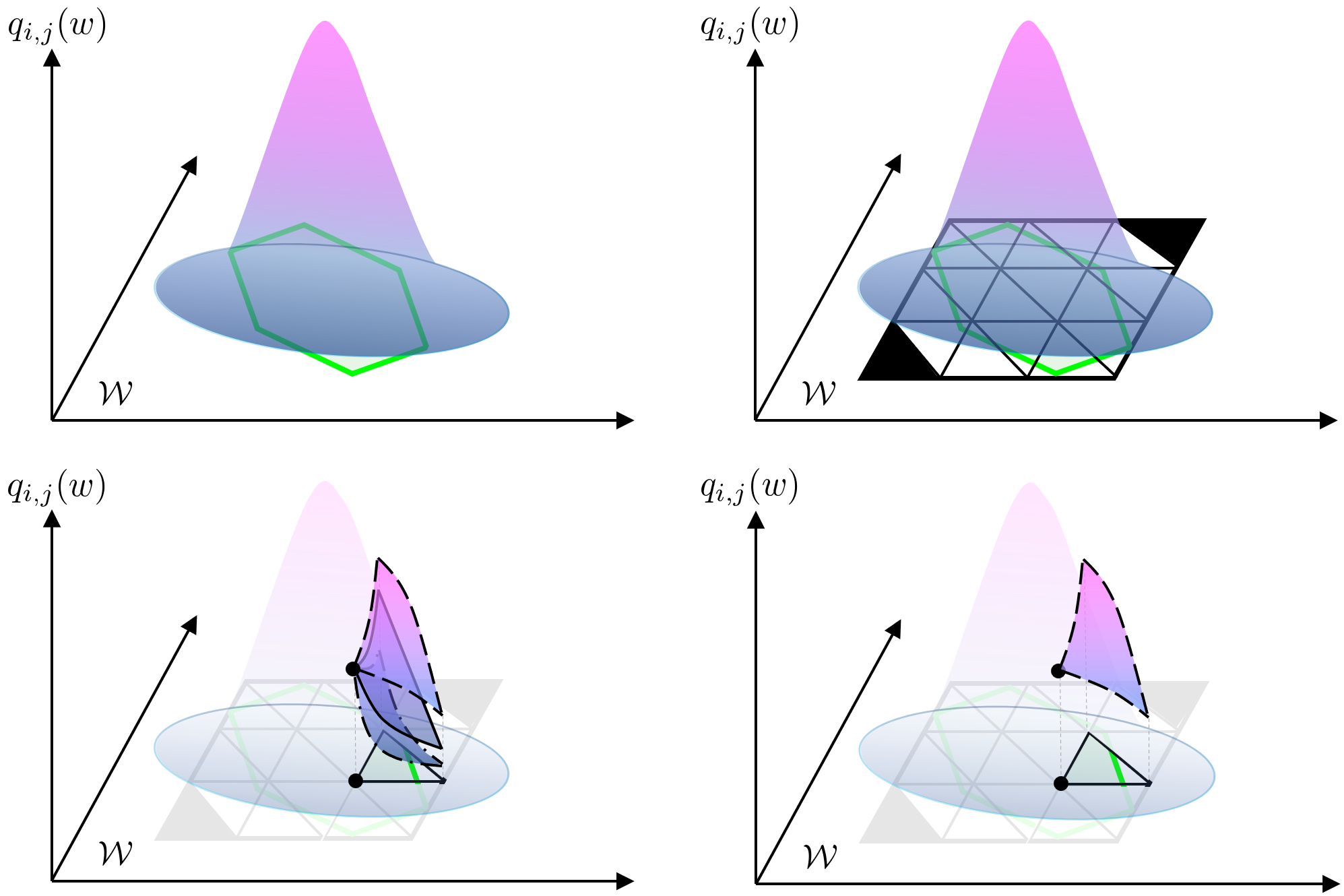}
         \caption{PDF Integration (Sec. \ref{subsubsec: PDF Integration}).}
         \label{fig:PDF Integration}
    \end{subfigure}
    \begin{subfigure}[b]{0.24\textwidth}
         \includegraphics[trim={27.5cm, 18.4cm, 0cm, 0cm},clip,width=\textwidth]{figures/Relaxation_v4.png}
         \caption{Domain Relaxation (Sec. \ref{subsubsec: Domain Relaxation}).}
         \label{fig: Domain Relaxation}
    \end{subfigure}
    \begin{subfigure}[b]{0.24\textwidth}
         \includegraphics[trim={0cm, 0.3cm, 27.9cm, 18.5cm},clip,width=\textwidth]{figures/Relaxation_v4.png}
         \caption{Integrand Relaxation (Sec. \ref{subsubsec: Integrand Relaxation}).}
         \label{fig: subsubsec: Integrand Relaxation}
    \end{subfigure}
    \begin{subfigure}[b]{0.24\textwidth}
         \includegraphics[trim={27.5cm, 0.2cm, 0cm, 18.5cm},clip,width=\textwidth]{figures/Relaxation_v4.png}
         \caption{Closed-form Computation (Sec. \ref{subsubsec: Closed-form Computation}).}
         \label{fig: Closed-form Computation}
    \end{subfigure}
    \caption{An illustration of chance constraint relaxation.
    Given arbitrary $(i,j)\in\I\times\J$, in (a) the risk of collision between the ego vehicle and the $i$-th obstacle during time interval $\T_j$ is relaxed as the integration of probability density function $\pdf$ (shown in purple and blue) over zonotope $\xi_j(z_0,p)\oplus\zonocg{0}{\Gobs}\subset\W$ (shown in green). 
    In (b), $\xi_j(z_0,p)\oplus\zonocg{0}{\Gobs}$ is over-approximated by a collection of right-angled triangles colored in white and black depending on if a triangle has a nontrivial intersection with $\xi_j(z_0,p)\oplus\zonocg{0}{\Gobs}$ or not.
    In (c), Interval Arithmetic is used to generate an over- and under-approximation of $\pdf$ over each right-angled triangle. 
    And in (d), the integration of the over-approximation of $\pdf$ over each right-angled triangle is computed in closed form.}
    \label{fig: chance constraint relaxation}
\end{figure*}
\section{An Implementable Alternative to \opt{}}
\label{sec: chance constraint}

This section describes an implementable alternative to \opt{} that can be solved rapidly.
In particular, we discuss how to relax the chance constraint in \opt{} in a conservative fashion.

\subsection{Chance Constraint Relaxation}
\label{subsec:chance constraint}

The chance constraint in \opt{} is relaxed conservatively as illustrated in Figure \ref{fig: chance constraint relaxation}.
First, the risk of collision during $[0,\tf]$ is over-approximated using Assumption \ref{ass: offline reachability}.
Then we relax the domain of integration into a collection of right-angled triangles and relax the PDF as a quadratic polynomial. 
Finally, we describe a closed-form equation for the relaxed integral.

\subsubsection{PDF Integration}
\label{subsubsec: PDF Integration}
Recall the ego vehicle's performance during the planning horizon is over-approximated by a collection of maps $\{\xi_j\}_{j\in\J}$, then the chance constraint in \opt{} is relaxed according to the following lemma which follows directly from the first property in Assumption \ref{ass: offline reachability}:
\begin{lem}
\label{lem: risk condition}
    Suppose the ego vehicle starts with initial condition $z_0\in\Z_0$ and control parameter $p\in\PP$.
    Let probability density functions $\{\pdf\}_{(i,j)\in\I\times\J}$ and matrix $\Gobs$ be as assumed in Assumption \ref{ass: obs pdf}.
    Let maps $\{\xi_j\}_{j\in\J}$ be as assumed in Assumption \ref{ass: offline reachability}.
    Then for arbitrary $\epsilon\in[0,1]$, \eqref{eq: chance constraint def} holds if
    \begin{equation}
    \label{ineq: pdf relaxation}
        \sum_{i\in\I}\sum_{j\in\J}\int_{\xi_j(z_0,p)\oplus\zonocg{0}{\Gobs}} \pdf(w) ~dw \leq \epsilon.
    \end{equation}
\end{lem}

\subsubsection{Domain Relaxation}
\label{subsubsec: Domain Relaxation}
Recall by Assumption \ref{ass: offline reachability} that $\xi_j(z_0,p)$ can be rewritten as $\zonocg{c_{j}(z_0)+A_{j}\cdot p}{G_{j}}=\zonocg{c_{j}(z_0)}{G_{j}}+A_{j}\cdot p$.
Then to relax the domain of integration $\xi_j(z_0,p)\oplus\zonocg{0}{\Gobs}$, we start by constructing a $k$-by-$k$ grid that covers zonotope $\xi_j(z_0,0)\oplus\zonocg{0}{\Gobs} = \zonocg{c_{j}(z_0)}{[G_{j},\Gobs]}$ where $k$ is some user-specified positive integer.
Each cell in the grid shares the same size and is indexed by its row and column index in the grid.
Each cell in the grid is further divided into two simplexes as right-angled triangles that are indexed by 1 or -1 corresponding to the lower or upper triangles of the cell, respectively.
For notational ease, denote $\SS_{j,k_1,k_2,k_3}(z_0)\subset \W$ as the simplex indexed by $k_3\in\{-1,1\}$ in the cell on the $k_1$-th row and $k_2$-th column of the grid that covers $\zonocg{c_{j}(z_0)}{[G_{j},\Gobs]}$.
Define  
\begin{align}
\label{eq: bT def}
        &\hspace{-0.1cm}\bT_{j}(z_0) := \big\{\SS_{j,k_1,k_2,k_3}(z_0) \mid k_1,k_2\in\{1,2,\ldots,k\}, \\
        &\hspace{0.3cm} |k_3|=1, ~ \SS_{j,k_1,k_2,k_3}(z_0)\cap\zonocg{c_{j}(z_0)}{[G_{j},\Gobs]} \neq\emptyset  \big\}. \nonumber
\end{align}
as the collection of every possible $\SS_{j,k_1,k_2,k_3}(z_0)$ that intersects with $\zonocg{c_{j}(z_0)}{[G_{j},\Gobs]}$, thus
\begin{equation}
\label{eq: zonotope covering}
 \zonocg{c_{j}(z_0)}{[G_{j},\Gobs]}\subset( \cup_{\SS\in\bT_{j}(z_0)}\SS).   
\end{equation}
Notice $\xi_j(z_0,p) = \xi_j(z_0,0) + A_{j}\cdot p$, therefore $\xi_j(z_0,p)\oplus\zonocg{0}{\Gobs} \subset \left( (\cup_{\SS\in\bT_{j}(z_0)}\SS) +A_{j}\cdot p \right)$ and
\begin{align}
    \begin{split}
        \int\displaylimits_{\xi_j(z_0,p)\oplus\zonocg{0}{\Gobs}} \hspace*{-0.75cm}  \pdf(w)~dw \leq \sum_{\SS\in\bT_{j}(z_0)} \int_{\SS +A_{j} p}  \pdf(w)~dw. \label{ineq: int relax on region}
    \end{split}
\end{align} 
\begin{rem}
    Note that for any $\SS\in\bT_{j}(z_0)$, $\SS$ depends on $z_0$ and $j$. 
    However, to reduce notational complexity, we drop its dependency on $z_0$ and $j$ in the remainder of this manuscript.
\end{rem}


\subsubsection{Integrand Relaxation}
\label{subsubsec: Integrand Relaxation}

Next, we present a lemma, whose proof is provided in Appendix \ref{app: proof of integrand relaxation}, to conservatively approximate the integral in the chance constraint by relaxing the integrand.
\begin{lem} \label{lem:integrandrelaxation}
 Suppose the ego vehicle starts with initial condition $z_0\in\Z_0$ and control parameter $p\in\PP$.
 Let $\SS\in\bT_{j}(z_0)$ and let $\wc_\SS$ denote the vertex of the right angle in $\SS$.
 Define the function $\pdfT: \W \times \PP \to [0,\infty)$ as
\begin{equation}
\begin{split}
\pdfT(w,p) &:= \pdf(\wc_\SS+A_{j}p)+\frac{\partial \pdf}{\partial w}(\wc_\SS+A_{j}p)\cdot\\
        &\cdot(w-\wc_\SS-A_{j}p)+\frac{1}{2}(w-\wc_\SS-A_{j}p)^\top\cdot\\
        &\cdot H_\SS\cdot(w-\wc_\SS-A_{j}p),
\end{split}
\end{equation}
where $H_\SS\in\R^{2\times 2}$ is generated by taking element-wise supremum of the Hessian of $\pdf$ over $\SS\oplus A_{j}\PP$ using Interval Arithmetic \cite{hickey2001interval}.
 Then for all $w\in\SS+A_{j}p$:
\begin{equation}
\label{eq: f after IA}
    \begin{split}
    \pdf(w) \leq \pdfT(w,p).
\end{split}
\end{equation}
\end{lem}
\noindent Note the inequality in \eqref{eq: f after IA} flips if $H_\SS$ is generated by taking an element-wise infimum of the Hessian of $\pdf$.
This would create an under-approximation to $\pdf$.


\subsubsection{Closed-form Computation}
\label{subsubsec: Closed-form Computation}
As a result of \eqref{ineq: int relax on region} and \eqref{eq: f after IA}, the following inequality holds:
\begin{equation}
    \begin{split}
        &\int_{\xi_j(z_0,p)\oplus\zonocg{0}{\Gobs}} \pdf(w)~dw\leq \\
        &\hspace{1.5cm} \leq \sum_{\SS\in\bT_{j}(z_0)} \int_{\SS +A_{j} p} \pdfT(w,p)~dw. \label{ineq: relaxed integration}
    \end{split}
\end{equation}
Notice that $\pdfT(w,p)$ defined in \eqref{eq: f after IA} is indeed a quadratic polynomial of $w$, and $\SS+A_{j}p$ is a simplex in $\R^2$.
One can then compute $\int_{\SS +A_{j} p} \pdfT(w,p)~dw$ in closed-form as follows:

\begin{thm}
\label{thm: lasserre}
    For any $i\in\I$ , $j\in\J$, $z_0\in\Z_0$, $p\in\PP$ and $\SS\in\bT_{j}(z_0)$, let $A_{j}\in\R^{2\times n_p}$ be defined as in Assumption \ref{ass: offline reachability}, let $\pdfT$ be defined as in \eqref{eq: f after IA}, and let $\idx_{k_3}(\SS)$ denote the last index of its argument (i.e., $\idx_{k_3}(\SS_{j,k_1,k_2,k_3}(z_0)) = k_3$).
    Assume that positive numbers $l_1$ and $l_2$ give the lengths of horizontal and vertical right angle sides of $\SS$ respectively, then
    \begin{equation}
    \label{eq: lasserre}
    \begin{split}
        \hspace{-0.34cm}\int_{\SS +A_{j} p} \pdfT&(w,p) dw =  \frac{1}{2\det(A_{\SS})} \bigg( \pdf(\wc_\SS+A_{j}p ) +\\
        &~+ \begin{bmatrix}\frac{1}{4\sqrt 3} & \frac{1}{4\sqrt 3}\end{bmatrix} \Big( \hat H_\SS \odot \begin{bmatrix}
        2 & 1\\ 1 & 2
        \end{bmatrix}  \Big)\begin{bmatrix}\frac{1}{2\sqrt 3} \\ \frac{1}{2\sqrt 3}\end{bmatrix}+\\
        &\hspace{1.3cm}+\frac{\partial \pdf}{\partial w}(\wc_\SS+A_{j}p)\cdot A_\SS^{-1}\cdot\begin{bmatrix}\frac{1}{3}\\\frac{1}{3}\end{bmatrix}\bigg)
    \end{split}
    \end{equation}
    where $\odot$ denotes the element-wise multiplication and 
    \begin{align}
        &A_\SS = \begin{bmatrix}
                \idx_{k_3}(\SS)/l_1 & 0 \\ 0 & \idx_{k_3}(\SS)/l_2
            \end{bmatrix},\\
        &\hat H_\SS = A_\SS^{-\top} H_\SS A_\SS^{-1}.
    \end{align}
\end{thm}

\begin{proof}
    The claim follows from \cite[Theorem 1.1]{lasserre2021simple} and the fact that $A_\SS\cdot\big((\SS+A_{j}p)-(\wc_\SS+A_{j}p)\big)$ equals the canonical simplex $\Delta:= \{(a,b)\in\R^2\mid a+b\leq 1, a\geq 0, b\geq 0\}$.
\end{proof}

\subsection{Tractable Online Optimization}
\label{subsec:enhanced online opt}
The computation in Section \ref{subsec:chance constraint} provides a tractable way to enforce not-at-fault behavior compared to the original chance constraint in \opt{}.
As a result, \riskmethodname{} solves the following optimization during online planning:
\begin{align*}
    \min_{p \in \P} & ~ \cost(z_0,p) \hspace{5cm} \optE\\
    \text{s.t.}
    & ~ \sum_{i\in\I}\sum_{j\in\J}\sum_{\SS\in\bT_{j}(z_0)}\int_{\SS +A_{j} p} \pdfT(w,p)~dw \leq\epsilon.
\end{align*}
\noindent\optE{} is a strengthened version of \opt{} because satisfaction of the chance constraint in \optE{} implies satisfaction of the chance constraint in \opt{} by Lemma \ref{lem: risk condition}, \eqref{eq: zonotope covering}, and \eqref{eq: f after IA}. 
Thus, the following lemma holds based on Definition \ref{defn:notatfault-risk}.
\begin{lem}
\label{lem: feasible gives safety}
If the ego vehicle applies any feasible solution, $p^*\in \PP$, of \optE{} beginning from $z_0\in\Z_0$ at $t = 0$, then it is not-at-fault with a risk of collision at most $\epsilon$ during $[0,\tf]$.
\end{lem}



\subsection{Constraint Gradient and Parallelization}
\label{subsec:gradient}
To improve the solving procedure of \optE{}, we provide the derivative of its chance constraint. 
To compute the gradient of the chance constraint in \optE{}, it suffices to compute the derivative of $\int_{\SS +A_{j} p} \pdfT(w,p) ~dw$ in \eqref{eq: lasserre} with respect to $p$.
Notice that $A_\SS$ and $\hat H_\SS$ are invariant over $\PP$, then
\begin{equation}
    \label{eq: d-lasserre}
    \begin{split}
        &\frac{\partial}{\partial p}\int_{\SS +A_{j} p} \pdfT(w,p) ~dw =  \\
        &\hspace{1.2cm} = \frac{1}{2\det(A_{\SS})}\cdot\bigg(\frac{\partial \pdf}{\partial w}(\wc_\SS+A_{j}p ) \cdot A_{j} + \\
        &\hspace{1.6cm} +\begin{bmatrix}\frac{1}{3} &\frac{1}{3}\end{bmatrix}A_\SS^{-\top}\Hess_{\pdf}(\wc_\SS+A_{j}p)\cdot A_{j} \bigg),
    \end{split}
\end{equation}
where $\Hess_{\pdf}(\wc_\SS+A_{j}p)$ is the evaluation of the Hessian of $\pdf$ at $\wc_\SS+A_{j}p$.


Notice the computations of $\int_{\SS +A_{j} p} \pdfT(w,p)~dw$ and its gradient can be parallelized for all $(i,j)\in\I\times\J$ and all $\SS\in\bT_{j}(z_0)$ as indicated in Algorithm \ref{alg: cuda}.
$\bT_j(z_0)$ is generated in Line 1 for all $j\in\J$.
The outer parallel for loop that starts at Line 2 iterates every element in $\I\times\J$, and the inner parallel for loop that starts at Line 3 iterates over each simplex in $\bT_{j}(z_0)$.
The integral and its gradient over each simplex are computed from Lines 4 to 5.
Then the chance constraint and its gradient in \optE{} can be computed as the summation of all computed $\int_{\SS +A_{j} p}\pdfT(w,p) ~dw$ and $\frac{\partial}{\partial p}\int_{\SS +A_{j} p} \pdfT(w,p) ~dw $ respectively. 


\begin{algorithm}[t]
    \caption{Chance Constraint Parallelization}
    \label{alg: cuda}
    \begin{algorithmic}[1]
        \REQUIRE $z_0\in\Z_0$, $p\in\PP$, $\{\pdf\}_{(i,j)\in\I\times\J}$, $\{\xi_j\}_{j\in\J}$
        \STATE \textbf{Generate} $\{\bT_{j}(z_0)\}_{j\in\J}$ as in \eqref{eq: bT def} using $\{\xi_j\}_{j\in\J}$
        \STATE \textbf{Parfor} $(i,j)\in\I\times\J$ \textbf{do}
            \STATE \quad \textbf{Parfor} $\SS\in \bT_{j}(z_0)$ \textbf{do}
            \STATE \quad \quad \textbf{Compute} $\int_{\SS +A_{j} p} \pdfT(w,p) ~dw $ as in \eqref{eq: lasserre}
            \STATE \quad \quad \textbf{Compute} $\frac{\partial}{\partial p}\int_{\SS +A_{j} p} \pdfT(w,p) ~dw $ as in \eqref{eq: d-lasserre}
        \STATE \quad \textbf{End Parfor}
        \STATE \textbf{End Parfor}
    \end{algorithmic}
\end{algorithm}

\subsection{Online Operation}
\label{subsec: online operation}
Algorithm \ref{alg: online} summarizes the online operation of \riskmethodname{}.
It begins by sensing and constructing predictions for obstacle locations as in Assumption \ref{ass: obs pdf} in Line 1.
\texttt{OnlineOpt} then solves \optE{} to search for a not-at-fault plan with a risk of collision at most $\epsilon$ in Line 2.
If \optE{} is infeasible, then \riskmethodname{} terminates planning in Line 3, otherwise it enters the planning loop in Line 4.
Within the loop, \riskmethodname{} first resets time to 0 in Line 5 and executes the entire driving maneuver corresponding to $p^*$ in Line 6.
Meanwhile, \texttt{SenseObstacles} updates the predictions for obstacle locations in Line 7 and \texttt{StatePrediction} predicts the ego vehicle state at $t=\tm$ as in Assumption \ref{assum:tplan} in Line 8.
In the case when the predicted vehicle state $z_0\notin\Z_0$, then \riskmethodname{} breaks the planning loop in Line 9.
Otherwise \optE{} is solved again using the updated obstacle information and $z_0$ in Line 10, and \riskmethodname{} breaks the planning loop if \optE{} is infeasible or takes longer than $\tplan$ to find a solution in Line 11.
Finally, once \riskmethodname{} leaves the planning loop, the contingency braking maneuver corresponding to $p^*$ is executed.
This braking maneuver by construction brings the vehicle to a stop.
Note that the risk of collision of this maneuver was already verified to be less than $\epsilon$ during the previous planning iteration.
Subsequently, \riskmethodname{} terminates planning in Line 13.
Note, we are able to obtain the following theorem by iteratively applying Lemma \ref{lem: feasible gives safety}:

\begin{thm}
Suppose the ego vehicle can sense and predict the surrounding obstacles as in Assumption \ref{ass: obs pdf}, and starts from rest with an initial condition $z_0\in\Z_0$ at $t=0$.
Then by performing planning and execution as in Algorithm \ref{alg: online}, the ego vehicle is not-at-fault with a risk of collision at most $\epsilon$ for all time.
\end{thm}

\begin{algorithm}[t]
    \caption{\riskmethodname{} Online Planning}
    \label{alg: online}
    \begin{algorithmic}[1]
        \REQUIRE $z_0\in\Z_0$ and $\epsilon\in[0,1]$
        \STATE \textbf{Initialize:} $\{\pdf\}_{(i,j)\in\I\times\J}\gets\texttt{SenseObstacles}()$
        \STATE \textbf{Try} $p^*\gets\texttt{OnlineOpt}(z_0,\{\pdf\}_{(i,j)\in\I\times\J},\epsilon)$
        \STATE \textbf{Catch} terminate planning
        \STATE \textbf{Loop:} // \textit{Line 6 executes simultaneously with Lines 7-11}
            \STATE \quad \textbf{Reset} $t$ to 0
            \STATE \quad \textbf{Execute} $p^*$ during $[0,\tm)$
            \STATE \quad $\{\pdf\}_{(i,j)\in\I\times\J}\gets\texttt{SenseObstacles}()$
            \STATE \quad $z_0\gets\texttt{StatePrediction}(z_0,p^*,\tm)$
            \STATE \quad \textbf{If} $z_0\notin \Z_0$, \textbf{then} break
            \STATE \quad \textbf{Try} $p^*\gets\texttt{OnlineOpt}(z_0,\{\pdf\}_{(i,j)\in\I\times\J},\epsilon)$
            \STATE \quad \textbf{Catch} break
        \STATE \textbf{End}
        \STATE \textbf{Execute} $p^*$ during $[\tm,\tf]$, \textbf{then} terminate planning
    \end{algorithmic}
\end{algorithm}

\section{Experiments and Results}
\label{sec:experiments}


All experiments are conducted in MATLAB R2023a on a Ubuntu 22.04 machine with an AMD Ryzen 9 5950X CPU, two NVIDIA RTX A6000 48GB GPUs, and 64GB RAM.
Parallelization is achieved using CUDA 12.1.
\riskmethodname{} invokes C++ for online planning using IPOPT.
Additional details on experimental setup can be found in Appendix \ref{subsec: experiment_appendix} 
An additional experiment detailing the single planning iteration performance of \riskmethodname{}, CCPBA, Cantelli MPC and REFINE can be found in Appendix \ref{subsec: single_plan_sim}.
An ablation study illustrating the importance of the analytical gradient and closed-form over-approximation of the risk of collision for real-time performance can be found in Appendix \ref{subsec: ablation study}.
Readers can find our implementation\footnote{\urlx{https://github.com/roahmlab/RADIUS}} and the video\footnote{\urlx{https://youtu.be/8eU9fiA39sE}} of simulations and hardware demos utilising \riskmethodname{} online.

\subsection{Tightness and Generality of Risk Approximation}
\label{subsec: canonical test}

For effective motion planning, we desire a tight over-approximation of the risk of collision.
Additionally, we want \riskmethodname{} to be able to generalize to arbitrary probability distributions as well to accommodate other uncertainty representations.
Thus, to evaluate the tightness of our over-approximation of $\int_{\xi_j(z_0,p)\oplus\zonocg{0}{\Gobs}} \pdf(w)~dw$, we compare the proposed method's approximation of the risk of collision against Monte-Carlo integration \cite[Chapter 4]{caflisch1998monte}, the Cantelli inequality \cite[Section IV B]{wang2020}, and the Chance-Constrained Parallel Bernstein Algorithm (CCPBA) \cite[Chapter 6]{SeanFast} on 9000 randomly generated test cases.

In each test case, $(i,j,z_0,p)$ is randomly chosen from $\I\times\J\times\Z_0\times\PP$.
Additionally, $\pdf$ is set to be the probability density function of either a randomly generated 2-dimensional (2D) Gaussian distribution, 2D Beta distribution \cite{olkin2015constructions} or a 2D Multimodal distribution.
Each type of distribution is evaluated in 3000 test cases.
Note that because CCPBA is unable to handle a distribution that is not Gaussian, we only evaluate \riskmethodname{} and the Cantelli inequality on these non-Gaussian probability distributions.
We perform the Monte-Carlo integration of $\pdf$ over $\xi_j(z_0,p)\oplus\zonocg{0}{\Gobs}$ with $10^6$ samples and treat that as the ground truth.

The risk approximation error, which is the difference between the Monte-Carlo integration and each method,
is illustrated in Table \ref{table: integration on exponential family}.
The results from Table \ref{table: integration on exponential family} show that \riskmethodname{} provides significantly tighter upper bounds to the ground truth than CCPBA and the Cantelli inequality for the tested probability distributions.
The associated average and maximum times to compute the over-approximation for each method are shown in Table \ref{table: computation times} of Appendix \ref{canonical appendix}.

\begin{table}[!tb]
    \centering
    \begin{tabular}{|c||c|c|c|}
    \hline 
    \multirow{2}{*}{\textbf{Method}} & \textbf{Gaussian Error} & \textbf{Beta Error}  & \textbf{Multimodal Error} \\
     & \textbf{(Mean, Max.)} & \textbf{(Mean, Max.)} & \textbf{(Mean, Max.)}  \\ \hline
    \riskmethodname{} & \textbf{(0.0073, 0.0523)} & \textbf{(0.0065, 0.0489)} & \textbf{(0.0079, 0.0262)} \\ \hline
    Cantelli &  (0.1774, 0.3420) & (0.2221, 0.5062) & (0.4734, 0.5851)\\ \hline
    CCPBA & (0.1881, 0.2149) & - & - \\ \hline
    \end{tabular}
    \caption{Results for the risk of collision estimation error of \riskmethodname{} and the Cantelli Inequality when the obstacle location is represented by three types of probability distributions. 
    Note that CCPBA can only handle Gaussian distributions so does not have results for the other types of distributions.
    }
    \label{table: integration on exponential family}
\end{table}

\subsection{Simulations}
\label{subsec:sim}

This section compares \riskmethodname{} to two state-of-the-art chance-constrained motion planning algorithms: CCPBA \cite{SeanFast} and Cantelli MPC \cite{wang2020}, and one state-of-the-art deterministic motion planning algorithm REFINE \cite{REFINE} in dense highway scenarios.
Additionally, we evaluate how varying the allowable risk threshold ($\epsilon$) for \riskmethodname{} affects the ego vehicle behavior in various unprotected left turn scenarios.
For all simulation experiments, $\tplan$ is set as 3[sec] and $\tf$ is chosen according to \cite[Lemma 14]{REFINE}.
A description of parameterized desired trajectories that are used in this work is provided in Appendix \ref{app: trajectory parametrization}.

\begin{figure*}[t]
    \centering
    \begin{subfigure}[b]{0.49\textwidth}
         \includegraphics[trim={0cm, 9.15cm, 8.4cm, 0cm},clip,width=\textwidth]{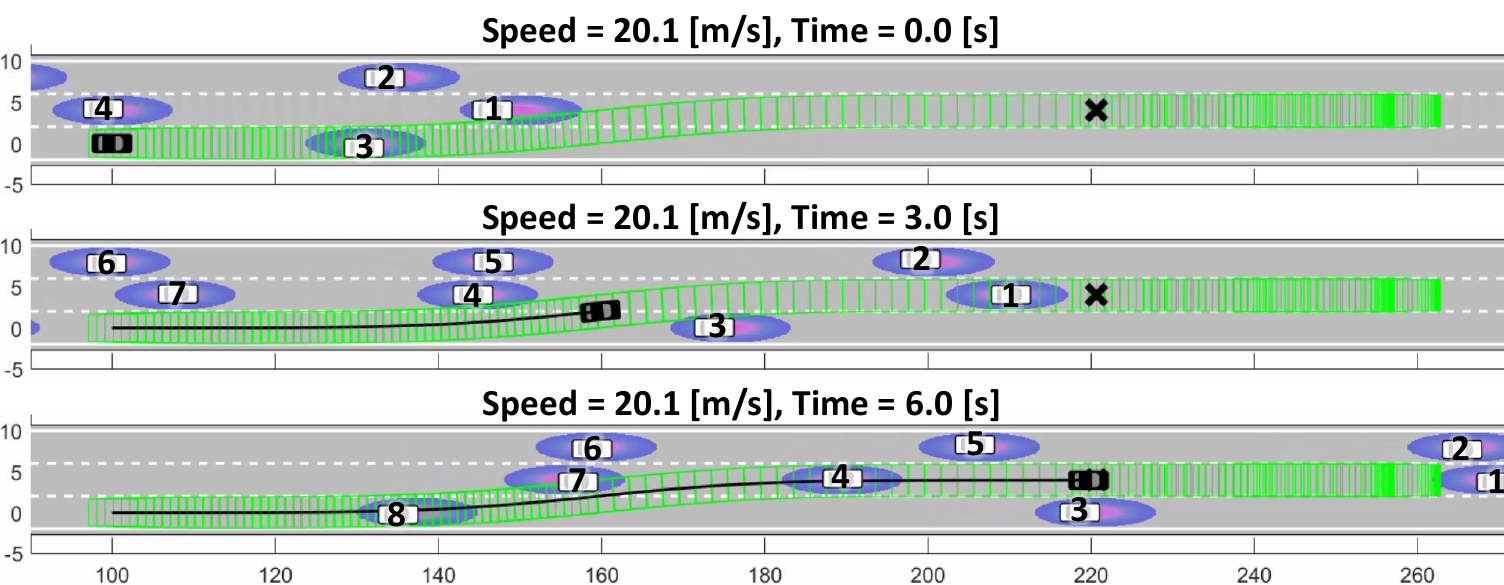}
         \caption{\riskmethodname{} utilized.}
         \label{fig: simulation result 1-riskrtd}
     \end{subfigure}
    \begin{subfigure}[b]{0.49\textwidth}
         \includegraphics[trim={0cm, 9.15cm, 8.4cm, 0cm},clip,width=\textwidth]{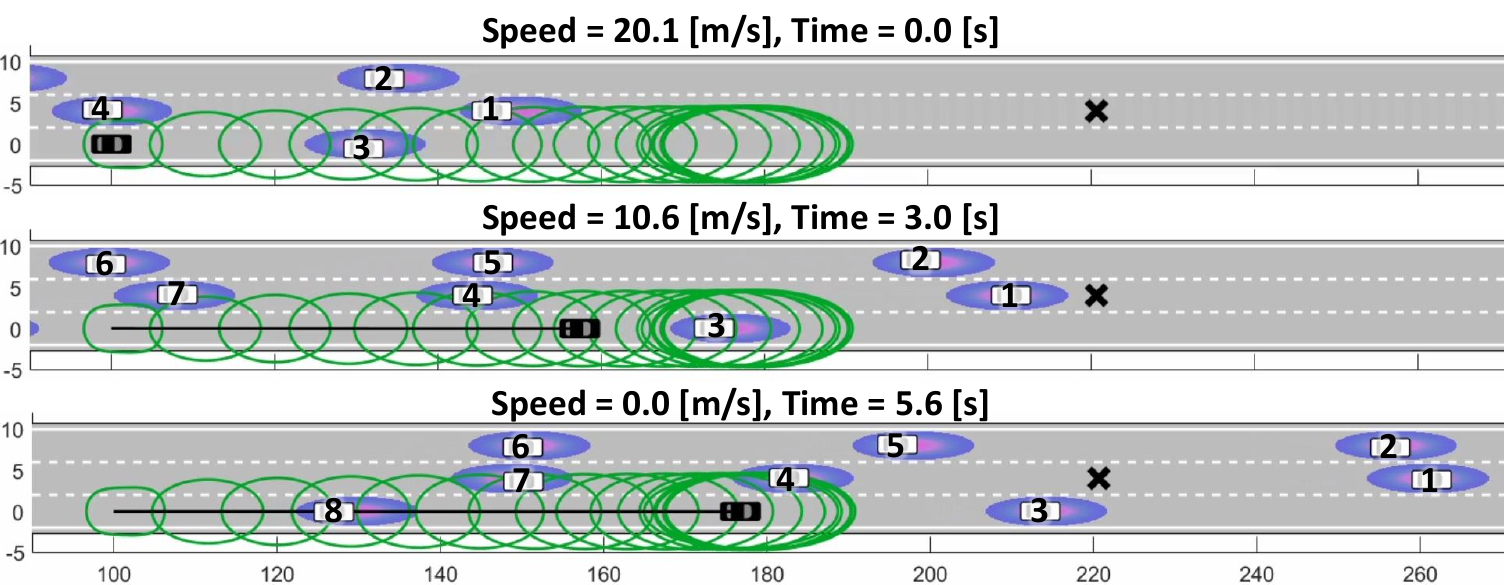}
         \caption{CCPBA utilized.}
         \label{fig: simulation result 1-CCPBA}
     \end{subfigure}
    \begin{subfigure}[b]{0.49\textwidth}
         \includegraphics[trim={0cm, 9.15cm, 8.4cm, 0cm},clip,width=\textwidth]{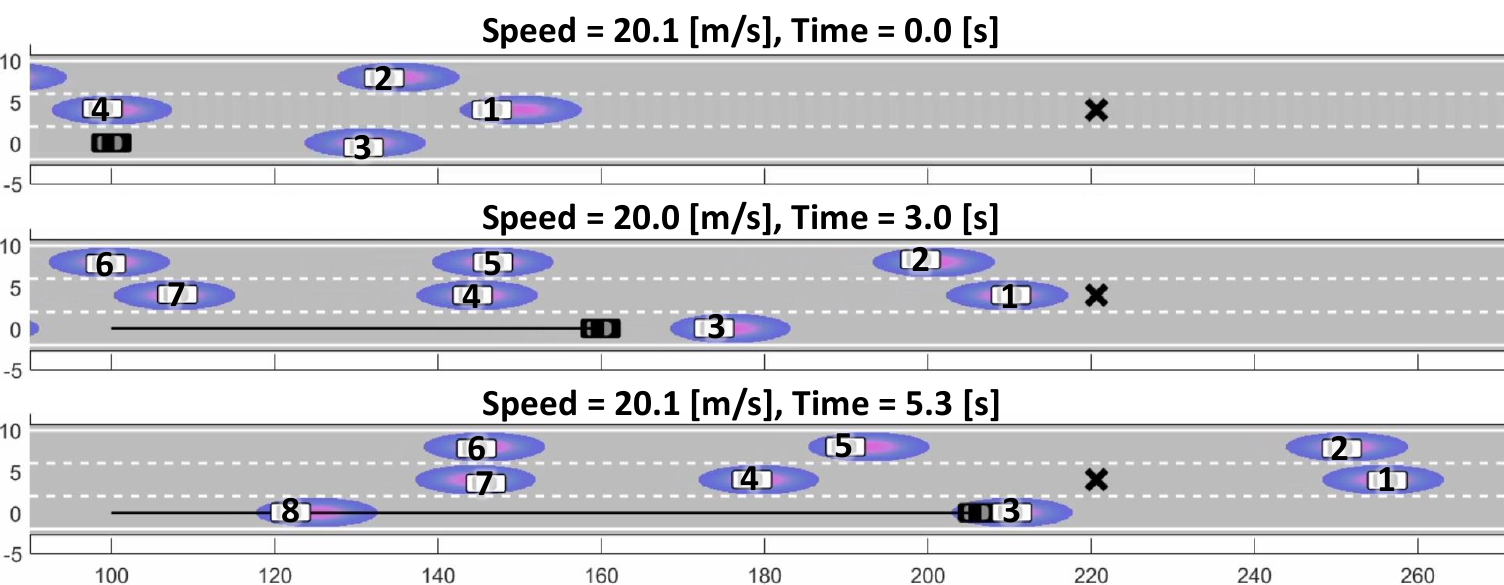}
         \caption{Cantelli MPC utilized.}
         \label{fig: simulation result 1-cantelli}
     \end{subfigure}
     \begin{subfigure}[b]{0.49\textwidth}
         \includegraphics[trim={0cm, 9.15cm, 8.4cm, 0cm},clip,width=\textwidth]{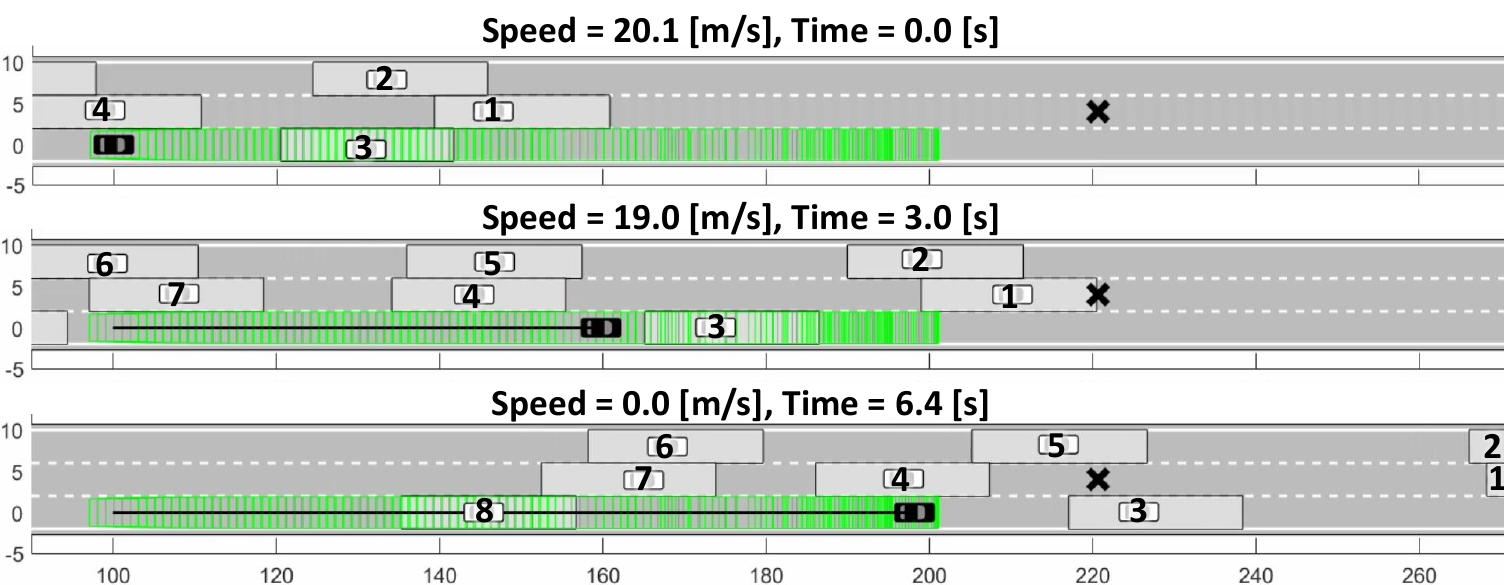}
         \caption{REFINE utilized.}
         \label{fig: simulation result 1-refine}
     \end{subfigure}
    \caption{Example of a single planning iteration of a simulated scenario in which \riskmethodname{} is able to navigate the ego vehicle (black) to the provided waypoint (black cross) through a lane change maneuver solved by one planning iteration, while CCPBA and REFINE execute contingency braking maneuvers as they are unable to find feasible solutions.
    Cantelli MPC results in a crash, as it has no fail-safe maneuver, so it instead maintains its velocity and continues searching for a feasible solution to execute the lane change.
    Forward reachable sets are shown in green.
    Obstacles are shown in white and are marked by their indices to make them trackable among different time instances.
    Probability distributions are illustrated as purple and blue ellipses where the probability density from low to high is illustrated from blue to purple.
    Note that for visualization purposes we have buffered the probability distributions shown in this figure by the footprints of the obstacles.
    Lastly, the convex hull that REFINE uses to over-approximate the 5-$\sig$ regions of the obstacle probability distributions are shown as gray boxes around the obstacles for all $(i,j) \in \I\times\J$. 
    Note that the distances on the x-axis and y-axis are in [m].}
    \label{fig: single planning iteration}
\end{figure*}

    
    

\subsubsection{3-Lane Highway}
\label{subsec: multiple_plan_sim}

This experiment evaluates the performance of \riskmethodname{} in a 3-lane highway environment where it must execute multiple planning iterations in succession and compares its performance to CCPBA, Cantelli MPC and REFINE.
We compare the methods over 1000 randomly generated 3-lane highway scenarios in simulation with $\epsilon=0.05$.
In each simulation scenario, the ego vehicle is expected to navigate through dynamic traffic for 1000[m] from a given initial position.  
Each scenario contains 3 static obstacles and a number of moving vehicles as dynamic obstacles, where the number of moving vehicles in each scenario is randomly selected between 5 and 25. 
To ensure that we can compare to CCPBA, we choose the $\pdf$ describing each obstacle's location during $\T_j$ as a Gaussian with a standard deviation of $\sig$ for any $(i,j)\in\I\times\J$.
To compare to REFINE, which is a deterministic motion planning algorithm that assumes perfect knowledge of the locations of obstacles, we require that REFINE avoids a box that over-approximates the 5-$\sig$ region of a Gaussian distribution.
Details on the reasoning behind the choice of 5-$\sig$ are explained in the last paragraph of Appendix \ref{subsec: sim_appendix}.

Each scenario is simulated for 10 trials resulting in a total of 10,000 simulation cases. 
In each trial, the starting locations of the ego vehicle and obstacles are the same, but the trajectories the obstacles follow are varied to capture the uncertain nature of each scenario.
Additional details on how these scenarios and obstacle trajectories are generated can be found in Appendix \ref{subsec: sim_appendix}.

Figure \ref{fig: single planning iteration} depicts an example of a single planning iteration of a scenario where \riskmethodname{} is able to successfully find a solution to execute a lane change to get to the given waypoint, while CCPBA, Cantelli MPC and REFINE are unable to execute the lane change.
Table \ref{table: multiplan highway result} presents \riskmethodname{}, CCPBA, Cantelli MPC, and REFINE's results for all 10,000 cases of the 3-lane highway multiple planning horizon experiment.
As seen in Table \ref{table: multiplan highway result}, \riskmethodname{} is able to achieve a higher success rate than CCPBA and Cantelli MPC. 
This is in part due to the fact that \riskmethodname{} is able to more closely approximate the ground truth risk of collision as shown in Table \ref{table: integration on exponential family}.
REFINE also has a lower success rate than \riskmethodname{} due to its conservative treatment of the uncertain region for the obstacle location.
Cantelli MPC has more crashes than CCPBA, and \riskmethodname{} because (1) Cantelli MPC, unlike CCPBA and \riskmethodname{} can only enforce risk-aware safety at discrete instances so crashes could occur in between these instances, (2) Cantelli MPC does not have a fail-safe stopping maneuver in case it cannot find a solution. 
Instead, it maintains speed and searches for a new solution and sometimes crashes into the vehicles in front of it.

Notice that despite the fact that $\epsilon=0.05$, \riskmethodname{} has a crash rate that is smaller than 5\%.
This occurs because even though \riskmethodname{} has a significantly tighter over-approximation than the comparison methods, it still over-approximates the true risk of collision.
This difference in the over-approximation and the true risk of collision can be attributed to: (1) the multiple relaxations made to compute the closed-form over-approximation highlighted in Section \ref{subsec:chance constraint}, and (2) the over-approximation of the actual ego vehicle location within the $j$-th time interval with the collection of maps introduced in Assumption \ref{ass: offline reachability}.
Future work will explore the contribution of each of these relaxations to the extent of over-approximation.
{\bf Note that despite this over-approximation, \riskmethodname{} is still significantly less conservative than deterministic algorithms like REFINE as seen in Table \ref{table: multiplan highway result}}.
\begin{table}[!tb]
    \centering
    \begin{tabular}{|c||c|c|c|c|}
    \hline 

    \multirow{2}{*}{\textbf{Method}} & \textbf{Success} & \textbf{Crash} & \textbf{Safely} & \textbf{Solve Time [s]} \\
    & \textbf{[\%]} & \textbf{[\%]} & \textbf{Stop [\%]}  & \textbf{(Mean, Max.)} \\\hline
    \riskmethodname{} & \textbf{80.3} & 0.8 & 18.9 & (0.412, 0.683) \\\hline
    CCPBA & 44.5 & 0.4 & 55.1 & \textbf{(0.291, 0.417)}  \\ \hline
    Cantelli MPC & 25.6 & 74.4 & 0.0 & (0.743, 3.921) \\\hline
    REFINE & 61.0 & \textbf{0.0} & 39.0 & (0.341, 0.562)  \\\hline
    
    \end{tabular}
    \caption{Simulation results comparing \riskmethodname{} to CCPBA, Cantelli MPC and REFINE on a 1000[m] stretch of dense highway. 
    A successful trial means the ego vehicle was able to successfully travel 1000[m] without crashing or coming to a stop.}
    
    \label{table: multiplan highway result}
\end{table}


\subsubsection{Left Turning}

 The aim of this experiment is to evaluate how different allowable risk thresholds ($\epsilon$) affect the behavior and performance of \riskmethodname{} over a single planning iteration, and unprotected left turns are ideal scenarios for this task. 
 In these scenarios, being less conservative has a large effect on the ability of the ego vehicle to complete the left turn faster by navigating through tight windows in the traffic of oncoming vehicles.
 In this experiment, the risk threshold is set to 0.01, 0.05, 0.10, 0.20 and 0.50, and the ego vehicle is tasked with navigating through 100 randomly generated unprotected left turning scenarios. 
 Each scenario is simulated for 10 trials, where in each trial the obstacle trajectories are varied in the same way as mentioned in Section \ref{subsec: multiple_plan_sim}.
 In each scenario, the ego vehicle is tasked with executing an unprotected left turn across two oncoming lanes at a 4-way intersection as depicted in Figure \ref{fig: left_turn}.
 Each scenario contains between 4 to 6 static obstacles occupying the horizontal lanes of the intersection and up to 4 dynamic obstacles that drive through the intersection in the vertical lanes, where the number of static and dynamic obstacles are randomly selected for each scenario.
 The starting lanes and initial positions of the dynamic obstacles are also randomly selected for each scenario.
 The initial speeds of the dynamic obstacles are randomly sampled from between 12[m/s] and 17[m/s].
 Similar to Section \ref{subsec: multiple_plan_sim}, $\pdf$ is chosen to describe an obstacle's location during $\T_j$ as a Gaussian for all $(i.j)\in\I\times\J$.
 
 Table \ref{table: left-turning} summarizes the results of the left turning experiment.
 As the allowable risk threshold $\epsilon$ increases, we see a reduction in the Average Time To Goal (ATTG) and Maximum Time To Goal (MTTG) as the ego vehicle is less conservative with higher $\epsilon$ and thus executes the unprotected left turns more quickly.
Additionally, as $\epsilon$ increases, the rate of collision also increases.

 \begin{figure}[!tb]
    \centering
    \includegraphics[trim={0cm, 7.cm, 10.cm, 0.2cm},clip,width=1\columnwidth]{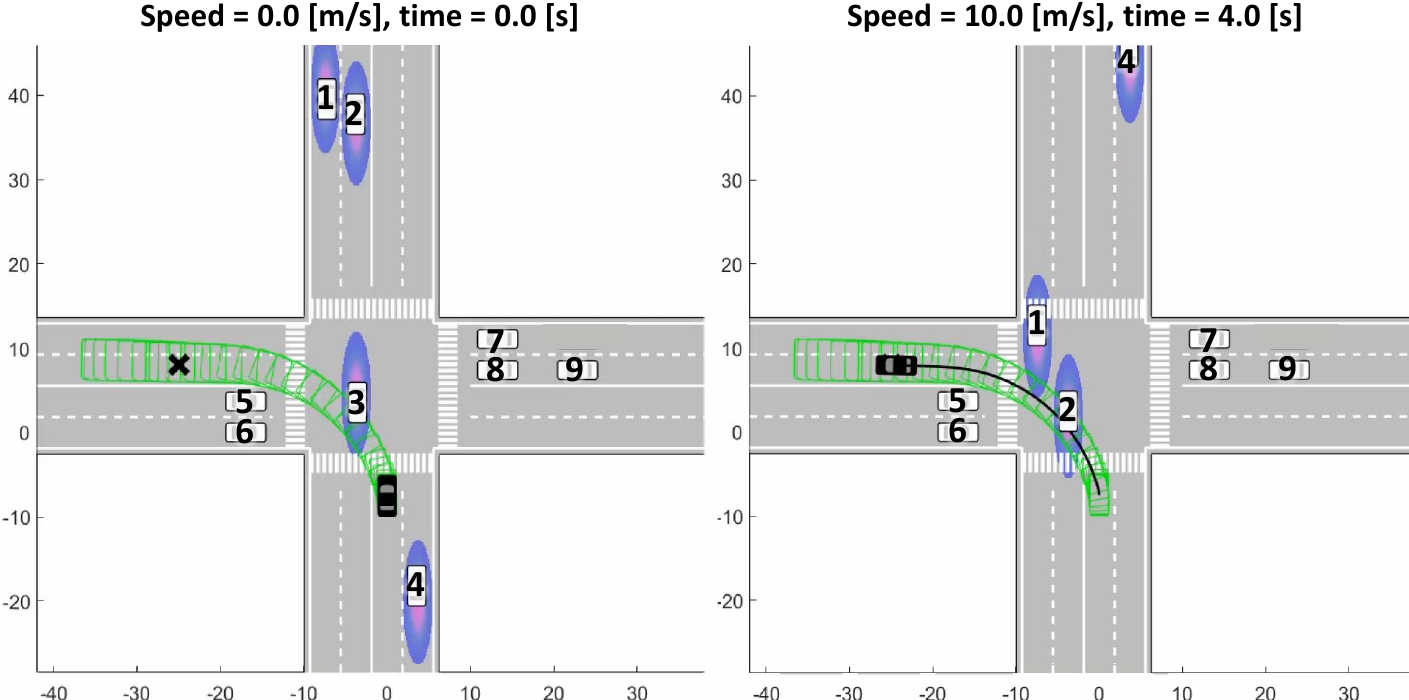}
    \caption{Example of a simulated unguarded left turning trial where the ego vehicle (black) navigates through a gap in the oncoming traffic to successfully execute the unguarded left turn.
    Forward reachable sets are shown in green.
    Obstacles are shown in white and are marked by their indices to make them trackable among different time instances.
    Probability distributions are illustrated as purple and blue ellipses where the probability density from low to high is illustrated from blue to purple.
    }
    \label{fig: left_turn}
\end{figure}
 
\begin{table}[!tb]
    \centering
    \begin{tabular}{|c||c|c|c|c|}
    \hline 

    \textbf{$\epsilon$} & \textbf{Success [\%]} & \textbf{Crash[\%]} & \textbf{ATTG [s]} & \textbf{MTTG [s]} \\ \hline
    0.01 & \textbf{100.0} & \textbf{0.0} & 8.868 & 13.500 \\\hline
    0.05 & 99.9 & 0.1 & 8.752 & 13.500\\ \hline
    0.10 & 99.5 & 0.5 & 8.187 & 13.400  \\\hline
    0.20 & 98.7 & 1.3 & 6.085 & 10.200 \\\hline
    0.50 & 95.6 & 4.4 & \textbf{5.891} & \textbf{10.000} \\\hline
    
    \end{tabular}
    \caption{Simulation results over a single planning iteration for \riskmethodname{} executing a left turn at different risk thresholds.
    Average Time To Goal (ATTG) and Maximum Time To Goal (MTTG) are only taken over successful trials.}
    
    \label{table: left-turning}
\end{table}

\subsection{Hardware}
To illustrate the capabilities of \riskmethodname{}, we also tested it on a $\frac{1}{10}$th-scale All-Wheel-Drive car-like robot, Rover, based on a Traxxas RC platform.
The first experiment shows \riskmethodname{} executing a lane change to overtake an obstacle Rover amidst uncertainty in the obstacle's location with varying risk thresholds. 
At lower risk thresholds \riskmethodname{} is more conservative and is not able to execute the lane change, whereas at higher thresholds \riskmethodname{} allows the ego Rover to aggressively overtake the obstacle Rover.
The second experiment illustrates a scenario where the obstacle Rover is driven laterally into the path of the ego Rover and it must generate a trajectory with a less than $\epsilon$ risk of collision into the obstacle Rover. 
At lower risk thresholds, \riskmethodname{} chooses to be conservative and triggers its contingency braking maneuver before it intersects with the obstacle Rover, unable to reach the set waypoint.
However, at higher risk thresholds \riskmethodname{} is more aggressive and executes an aggressive lane change in order to get to the waypoint.

\section{Extensions To Other Systems}
\label{sec: system_extensions}

The proposed algorithm, \riskmethodname{}, can generalize to robotic systems other than the vehicle model highlighted in this work, including manipulators and quadrotors. 
Specifically, \riskmethodname{} can be applied to any system that satisfies the following assumptions:
1) The system dynamics must have bounded modeling error (Assumption \ref{ass: dyn error bnd}); 
2) The system must be able to generate a motion plan in $\tplan$ seconds (Assumption \ref{assum:tplan});
3) The given system must have a large enough sensor radius as described in Assumption \ref{ass: sense horizon};
4) Obstacles must be represented as twice-differentiable probability density functions (Assumption \ref{ass: obs pdf});
5) There must exist a set of over-approximative zonotope reachable sets of a parameterized dynamical model for the system (Assumption \ref{ass: offline reachability}). 
Note that this last assumption is a critical, robot-specific assumption for RADIUS. 
However, zonotope reachable sets satisfying Assumption \ref{ass: offline reachability} have been constructed for a number of robotic systems including autonomous vehicles \cite{REFINE}, manipulators \cite{ARMOUR} and quadrotors \cite{quadrotor_kousik}.
\section{Conclusion}
\label{sec:conclusion}

This work proposes \riskmethodname{} as a real-time, risk-aware trajectory planner for autonomous vehicles operating in uncertain and dynamic environments.
The method proposes a novel method for computing a tight, differentiable, closed-form over-approximation of the risk of collision with an obstacle.
Given some allowable risk threshold ($\epsilon$), \riskmethodname{} uses this over-approximation conjunction with reachability-based methods to generate motion plans that have a no greater than $\epsilon$ risk of collision with any obstacles for the duration of the planned trajectory.
Furthermore, \riskmethodname{} is able to outperform the existing state-of-the-art chance-constrained and deterministic approaches in a variety of driving scenarios while achieving real-time performance.







\renewcommand{\bibfont}{\normalfont\small}
{\renewcommand{\markboth}[2]{}
\printbibliography}

\pagebreak

 \appendices
\section{Proof of Lemma \ref{lem:integrandrelaxation}}
\label{app: proof of integrand relaxation}
\begin{proof}
To relax the integrand $\pdf$, we start by computing the 2nd-order Taylor Expansion of $\pdf$ centered at $\wc_\T+A_{j}p$ and applying Mean Value Theorem (MVT) \cite[Theorem 4.1]{sahoo1998mean} to eliminate higher order terms in the Taylor Expansion as:
\begin{equation}
\label{eq: Taylor + MVT}
\begin{split}
    \hspace{-0.2cm}\pdf(w) = &\pdf(\wc_\SS+A_{j}p)+\frac{\partial \pdf}{\partial w}(\wc_\SS+A_{j}p)\cdot\\
        &\cdot(w-\wc_\SS-A_{j}p)+\frac{1}{2}(w-\wc_\SS-A_{j}p)^\top\cdot\\
        &\cdot\Hess_{\pdf}(w')\cdot(w-\wc_\SS-A_{j}p)
\end{split}
\end{equation}
where $w'\in\SS+A_{j}p$ is some point on the line segment joining points $\wc_\SS+A_jp$ and $w$ in $\SS+A_jp$, and $\Hess_{\pdf}$ gives the Hessian of $\pdf$.
Because MVT does not provide $w'$ in a closed-form, we then drop the dependency of $w'$ in \eqref{eq: Taylor + MVT} by bounding the Hessian of $\pdf$ as a matrix $H_\SS\in\R^{2\times 2}$ where $H_\SS$ is generated by taking element-wise supremum of $\Hess_{\pdf}$ over $\SS\oplus A_{j}\PP$ using Interval Arithmetic \cite{hickey2001interval}.
Note by construction of $\SS$ and $\wc_\SS$, it is guaranteed that either $w\geq\wc_\SS+A_{j}p$ for all $w\in\SS+A_{j}p$ or $w\leq\wc_\SS+A_{j}p$ for all $w\in\SS+A_{j}p$, thus by definition of $H_\SS$ the desired inequality holds for all $w\in\SS+A_{j}p$.
\end{proof}
\section{Additional Details}
\label{subsec: experiment_appendix}
Specifications of the full-size FWD vehicle and the Rover robot used in this work can be found in \cite[Section IX] {REFINE}, from which we also adapt control gains and zonotope reachable sets.

\subsection{Desired Trajectories}
\label{app: trajectory parametrization}
In this work we select 4 families of desired trajectories for driving maneuvers often observed during daily driving: speed changes, direction changes, lane changes and left turning.
Each desired trajectory is the concatenation of a driving maneuver and a contingency braking maneuver. 
Note that the duration $\tm$ of driving maneuvers remains constant within each trajectory family, but can vary among different trajectory families.
Each desired trajectory is parameterized by $p=[p_u,p_y]^\top\in\PP\subset\R^2$ where $p_u$ and $p_y$ decide desired longitudinal speed and lateral displacement respectively. 
The trajectory families associated with speed change, direction change and lane change are detailed in \cite[Section IX-A]{REFINE}, thus we describe desired trajectories that achieve left turning here.

To achieve a left turning maneuver, we set the desired trajectory for longitudinal speed as
\begin{equation}
    \udes(t,p) = \begin{dcases}
        \frac{11}{2}t,      \hspace{1.15cm} \text{ if } 0\leq t<\frac{1}{4}\tm \text{ and } \frac{11}{8}\tm< p_u\\
        p_u,        \hspace{1.35cm} \text{ if } 0\leq t<\frac{1}{4}\tm \text{ and } \frac{11}{8}\tm\geq p_u\\
        p_u,       \hspace{1.35cm} \text{ if } \frac{1}{4}\tm\leq t<\tm\\
        \ubrk(t,p),\hspace{0.2cm} \text { if } t\geq\tm
    \end{dcases}
\end{equation}
where $\ubrk:[\tm,\tf]\times\PP\rightarrow\R$ describes the desired trajectory of longitudinal speed during contingency braking and shares the same formulation with the other 3 trajectory families. 
The desired trajectory of the yaw rate is as follows: 
\begin{equation}
    \rdes(t,p) = \begin{dcases}
        \frac{1}{2}p_y\left(1-\cos\left(\frac{4\pi}{\tm} t\right)\right),      \hspace{0.1cm} \text{ if } 0\leq t<\frac{1}{4}\tm\\
        p_y,       \hspace{3.25cm} \text{ if } \frac{1}{4}\tm\leq t<\frac{3}{4}\tm\\
        \frac{1}{2}p_y\left(1-\cos\left(\frac{4\pi}{\tm} t\right)\right),      \hspace{0.1cm} \text{ if } \frac{3}{4}\tm\leq t<\tm\\
        0,\hspace{3.42cm} \text { if } t\geq\tm
    \end{dcases}
\end{equation}
And desired trajectory of heading is set as $\hdes(t,p) = h_0 + \int_0^t\rdes(\tau,p)~d\tau$ for any $t\in[0,\tf]$ and $p\in\PP$, where $h_0\in[-\pi,\pi]$ is the initial heading of the ego vehicle at time 0.

The duration $\tm$ of driving maneuvers for every trajectory family is 3[s] for speed change, 3[s] for direction change, 6[s] for lane change, and 4[s] for left change.
Because we do not know which desired trajectory ensures not-at-fault \textit{a priori}, to guarantee real-time performance, $\tplan$ should be no greater than the smallest duration of a driving maneuver.
Therefore in this work we set $\tplan = 3$[s].

\subsection{Computational Efficiency of Risk Approximation}
\label{canonical appendix}
In this section we add to the evaluation presented in Section \ref{subsec: canonical test} of the main text.
We compare the computation times of the proposed method's approximation of the risk of collision against the Cantelli inequality, and the Chance-Constrained Parallel Bernstein Algorithm (CCPBA) on the same 3000 randomly generated test cases.
In each test case, $(i,j,z_0,p)$ is randomly chosen from $\I\times\J\times\Z_0\times\PP$. 
Table \ref{table: computation times} summarizes the computation times for generating the risk of collision by approximating the following the integral:$\int_{\xi_j(z_0,p)\oplus\zonocg{0}{\Gobs}} \pdf(w)~dw$ when $\pdf$ is modeled as a Gaussian.
Table \ref{table: computation times} shows that \riskmethodname{} is able to compute the approximation of the risk of collision about 60\% faster than CCPBA, but not as fast as evaluating the Cantelli inequality.

\begin{table}[h]
    \centering
    \begin{tabular}{|c||c|c|}
    \hline 
    \multirow{2}{*}{\textbf{Method}} & \textbf{Mean} & \textbf{Maximum} \\
     & \textbf{Solve Time [ms]} & \textbf{Solve Time [ms]}\\\hline
     \riskmethodname{} & 0.4372 & \textbf{2.3150} \\\hline
      Cantelli & \textbf{0.0181} & 2.5223\\\hline
    CCPBA & 1.1012 & 3.5820\\\hline
    \end{tabular}
    \caption{Results comparing the computation times for generating an approximation for the risk of collision using \riskmethodname{}, CCPBA, and the Cantelli Inequality when the obstacle uncertainty is represented as a Gaussian distribution.}
    \label{table: computation times}
\end{table}

\subsection{Additional Simulation details}
\label{subsec: sim_appendix}

We discuss more in detail the experimental setup for the simulation experiments discussed in Section \ref{subsec:sim}.
In both the 3-lane highway and left turning scenarios, the uncertainty in the locations and predictions of obstacle motions for each obstacle is represented as a Gaussian distribution $\pdf$ for the stochastic methods for any $(i,j)\in\I\times\J$.
More exactly, $\pdf$ is a Gaussian distribution that represents the possible location of the $i$-th obstacle during the $j$-th time interval $\T_j$, and is constructed such that its mean, $\pdfmu$, is at the center of a lane for all $\T_j$.
As time increases we translate this uncertain region forward at a constant speed while keeping it centered in the lane.
In other words, $\pdfmu - \mu_{i,j-1}$ is constant for all $j>1$ in $J$.
Recall, from Definition \ref{def: footprint}, that $L$ and $W$ are the length and width of each obstacle, respectively.
Then for any $(i,j)\in\I\times\J$, the standard deviation $\sigma_{i,j}$ of $\pdf$ is chosen such that the $3\sigma_{i,j}$-region of the Gaussian distribution covers the area of width $3.7-W$ and length $L+u_i(0)\cdot \Delta_t$, where $u_{i}(0)$ is the speed of the $i$-th obstacle at time 0 and $\Delta_t$ is the time interval defined in Section \ref{subsec: environment}.
Note the choice of width $3.7-W$ ensures that the footprint of the obstacle always stays inside the lane. 

To capture the stochastic nature of each scenario, we simulate each scenario for 10 trials.
For each given scenario and for any $i\in\I$, the $i$-th obstacle is always initialized with the same state, but follows different trajectories among different trials of this scenario. 
These trajectories are selected such that the locations of the $i$-th obstacle during the $j$-th time interval $\T_j$ are randomly sampled from $\pdf$ for each trial, where $\pdf$ is kept constant for all trials of the same scenario for any $(i,j)\in\I\times\J$.
This variability in the trajectories allows us to capture trials where a specific obstacle trajectory may cause a crash, whereas in other trials of the same scenario, different trajectories do not cause crashes. 

Because REFINE is a deterministic motion planning algorithm, it assumes perfect knowledge of the locations of obstacles.
In scenarios where there is uncertainty in the location of obstacles, deterministic algorithms like REFINE often generate motion plans to avoid the entire uncertain region.
However, because Gaussian distributions have non-trivial probability density over the entire space $\W$, applying REFINE naively would require avoiding the entire world space.
As such, we require REFINE to avoid a box that over approximates the 5-$\sigma_{i,j}$ region of the Gaussian distribution for any $(i,j)\in\I\times\J$ as shown in Figure \ref{fig: simulation result 1-refine}.
Note such 5-$\sigma_{i,j}$ region accounts for 99.99999\% of the probability mass.

\section{Additional Experiments}
\label{sec:additional_experiments}

In addition to the experiments reported in the main body of this work, we perform two additional experiments to evaluate the proposed method. 

\subsection{Single Planning Horizon 3-Lane Highway Environment}
\label{subsec: single_plan_sim}

This experiment compares the performance of \riskmethodname{} to CCPBA, Cantelli MPC, and REFINE in a 3-lane highway driving scenario in simulation over 10 randomly generated scenarios. 
The aim of this experiment is to see how \riskmethodname{} performs in comparison to the other methods when trying to execute a single lane change in a small stretch of dense highway.
Each scenario contains 1 static obstacle and between 4 to 9 dynamic obstacles, where the number of dynamic obstacles was randomly selected for each scenario.  
It is important to note that in these scenarios, since it was only over a single planning iteration, these obstacles were all spawned within a 100[m] radius of the ego vehicle.
This resulted in scenarios that are denser than the situations typically encountered in the 3-lane experiment described in Section \ref{subsec: multiple_plan_sim}.
The initial speeds of all dynamic obstacles are also randomly sampled from between 15[m/s] to 20[m/s] in each scenario.
In each of these scenarios, the ego vehicle is randomly initialized in a lane, and is commanded to get as close as possible to a given waypoint in a different lane.
The ego vehicle is expected to achieve the task in a single planning iteration with allowable risk threshold $\epsilon = 0.05$ and $\tplan = 3$[sec].
Success in this experiment is characterized by being able to find a feasible lane change maneuver during the planning time and execute it without crashing.
Each scenario is simulated for 200 trials, resulting in a total of 2000 simulation cases.
The obstacle uncertainties for the stochastic methods and REFINE are represented in the same way as in Section \ref{subsec: multiple_plan_sim}.
Table \ref{table: singleplan highway result} summarizes \riskmethodname{}, CCPBA, Cantelli MPC, and REFINE's performance on the 3-lane single planning iteration scenarios.
In this experiment we see that \riskmethodname{} was able to successfully execute this lane change maneuver more often than the comparisons.

\begin{table}[!tb]
    \centering
    \begin{tabular}{|c||c|c|c|c|}
    \hline 
    \multirow{2}{*}{\textbf{Method}} & \textbf{Success} & \textbf{Crash} & \textbf{Other} \\ 
    & \textbf{[\%]} & \textbf{[\%]} & \textbf{Action [\%]}  \\\hline
    \riskmethodname{} & \textbf{81.8} & \textbf{0.0} & \textbf{18.2}  \\\hline
    CCPBA & 55.2 & \textbf{0.0} & 44.8  \\ \hline
    Cantelli MPC & 9.1 & 35.4 & 55.5 \\\hline
    REFINE & 21.2 & 0.0 & 78.8 \\\hline
    
    \end{tabular}
    \caption{Single planning iteration results using \riskmethodname{}, CCPBA and Cantelli MPC. ``Success" encompasses the trials where each method was able to successfully execute the lane change. 
    ``Other Action" encompasses the trials where a method did not complete the lane change maneuver, but instead either executed a safe stop maneuver or decided to keep driving in lane.}
    
    \label{table: singleplan highway result}
\end{table}

\subsection{Ablation Study}
\label{subsec: ablation study}

This section performs an ablation study to evaluate the effectiveness of each component of \riskmethodname{}.
In this ablation study, we examine the performance of 3 algorithms, each of which has a certain aspect of \riskmethodname{} removed.
Namely, the 3 algorithms are: \riskmethodname{} (No Analytical Gradient), Monte-Carlo RTD (Discrete) and Monte-Carlo RTD (Continuous).
For \riskmethodname{} (No Analytical Gradient), we examine the performance of \riskmethodname{} without providing the constraint gradient described in Section \ref{subsec:gradient}. 
Instead, we allow IPOPT to compute a numerical gradient during the optimization.
In the other 2 algorithms, instead of leveraging our over-approximation to the risk of collision from Section \ref{sec: chance constraint}, we instead replace this with a Monte Carlo style sampling method to estimate the risk of collision as defined in \eqref{ineq: pdf relaxation}.
For Monte-Carlo RTD (Discrete), to find a viable control parameter, we sample discrete points from the control parameter space $\P$ in a similar fashion to \cite{manzinger2020using}.
To approximate the risk of collision in a Monte-Carlo fashion, we sample 100 points from $\pdf$ and count the number of points that fall in the zonotope $\xi_j(z_0,p)\oplus\zonocg{0}{\Gobs}$ for any $(i,j)\in\I\times\J$, where $p$ corresponds to a sampled control parameter.
We then sum up these integrals over all $i \in \I$ and $j \in \J$ to estimate the risk of collision with \eqref{ineq: pdf relaxation}. 
Monte-Carlo RTD (Continuous) estimates the risk of collision in the same way as Monte-Carlo RTD (Discrete), however it searches for feasible not-at-fault plans over the continuous control parameter space $\P$ and requires IPOPT to compute a numerical gradient during only planning.

We evaluate these algorithms in another set of 10 randomly generated 3-lane highway scenarios, with the same experimental setup from \ref{subsec: single_plan_sim}.
Each scenario is simulated for 50 trials resulting in a total of 500 test cases.
The results of the ablation study can be found in Table \ref{table: ablation study}.
From the ablation study, it is evident that replacing the closed-form over-approximation with Monte-Carlo sampling methods causes an 8-10$\times$ increase in the solve time due to a longer risk of collision estimation time.
Additionally\new{,} both Monte-Carlo methods achieve a roughly 20\% lower success rate than \riskmethodname{}, which is too slow to be able to achieve real-time performance.
We note that increasing the number of samples from 100 would likely increase the success rate as the estimation of the risk of collision would get more accurate. 
However, this increase in the number of samples will cause the solve time to be even slower.
Without the analytical gradients being provided, \riskmethodname{}(NAG) requires an average time of about 10$\times$ slower than \riskmethodname{}.
It is evident from these results that both the analytical gradient and the closed-form over-approximation play important roles in being able to generate motion plans in real time.

\begin{table}[!tb]
    \centering
    \begin{tabular}{|c||c|c|c|c|}
    \hline 

     \multirow{2}{*}{\textbf{Algorithm}} &  \textbf{Success} &  \textbf{Crash} & \textbf{Other} & \textbf{Solve} \\
      & \textbf{[\%]}  & \textbf{[\%]} & \textbf{Action [\%]} & \textbf{Time [s]} \\ \hline
    MCRTD(D) & 59.0 & 0.0 & 41.0 & 4.781 \\ \hline
    MCRTD(C) & 60.0 & 0.0 & 40.0 & 3.334  \\\hline
    \riskmethodname{}(NAG) & \textbf{80.0} & 0.0 & \textbf{20.0} & 4.931 \\\hline
    \riskmethodname{} & \textbf{80.0} & 0.0 & \textbf{20.0} & \textbf{0.412}  \\\hline
    
    \end{tabular}
    \caption{Single planning iteration results using the \riskmethodname{} variants considered in the ablation study.
    ``Other Action" encompasses the trials where each method did not complete the lane change maneuver, but instead either executed a safe stop maneuver or decided to keep driving in the lane. \riskmethodname{} (No Analytical Gradient) is abbreviated as \riskmethodname{} (NAG), Monte-Carlo RTD (Discrete) is abbreviated as MCRTD(D) and Monte-Carlo RTD (Continuous) is abbreviated as MCRTD(C).
    }
    
    \label{table: ablation study}
\end{table}

\end{document}